\documentclass[twoside]{article}
\pdfoutput=1

\let\ifarxiv\iftrue

\ifarxiv
  \usepackage[accepted]{aistats2018}

\else
  \usepackage[accepted]{aistats2018}

\fi

\makeatletter
\def\paragraph{\@startsection{paragraph}{4}{\z@}{0pt}{-1em}{\normalsize\bf}}
\makeatother

\usepackage[utf8]{inputenc}
\usepackage[T1]{fontenc}
\usepackage[american]{babel}
\usepackage{url}
\usepackage{booktabs}
\usepackage{amsfonts}
\usepackage{nicefrac}
\usepackage{microtype}
\usepackage{csquotes}
\usepackage{enumitem}
\usepackage{etoolbox}

\usepackage{amstext}
\usepackage{amsmath}
\usepackage{amssymb}
\usepackage{amsthm}

\usepackage{todonotes}
\usepackage{appendix}
\usepackage{graphicx}
\usepackage{epstopdf}

\usepackage{caption}
\usepackage{subcaption}

\usepackage[colorlinks,linkcolor={red!80!black},citecolor={green!80!black},urlcolor={blue!80!black},hypertexnames=false]{hyperref}

\usepackage[noabbrev,capitalize]{cleveref}
\crefformat{equation}{(#2#1#3)}
\crefrangeformat{equation}{(#3#1#4) to~(#5#2#6)}
\crefmultiformat{equation}{(#2#1#3)}{ and~(#2#1#3)}{, (#2#1#3)}{ and~(#2#1#3)}
\crefname{appsec}{Appendix}{Appendices}

\usepackage{autonum}

\PassOptionsToPackage{firstinits}{biblatex}

\usepackage[backend=biber,style=authoryear,maxcitenames=2,maxbibnames=8,uniquelist=false,uniquename=init,natbib,dashed=false,isbn=false,doi=false]{biblatex}

\defbibheading{bibliography}[\refname]{\subsubsection*{#1}}  %
\DeclareMathOperator*{\argmin}{argmin}
\DeclareMathOperator{\diag}{diag}
\newcommand{\ud}{\mathrm d}
\DeclareMathOperator{\E}{\mathbb E}

\newcommand{\f}{\mathcal F}
\newcommand{\h}{\mathcal H}
\newcommand{\hs}{\mathrm{HS}}
\DeclareMathOperator{\N}{\mathcal N}
\newcommand{\bigO}{\mathcal O}
\newcommand{\bigOmega}{\Omega}
\newcommand{\littleO}{o}
\newcommand{\p}{\mathcal P}
\newcommand{\R}{\mathbb R}
\DeclareMathOperator{\range}{range}
\DeclareMathOperator{\spn}{span}
\DeclareMathOperator{\supp}{supp}
\DeclareMathOperator{\Tr}{Tr}
\newcommand{\tp}{^\mathsf{T}}
\newcommand{\X}{\mathcal X}
\newcommand{\lite}{\textsf{lite}}
\newcommand{\dae}{\textsf{dae}}
\newcommand{\full}{\textsf{full}}
\newcommand{\nystrom}{\textsf{nyström}}
\newcommand{\nystromd}{\textsf{nyström D}}
\newcommand{\grid}{\texttt{grid}}
\newcommand{\ring}{\texttt{ring}}

\newcommand{\httpsurl}[1]{\href{https://#1}{\nolinkurl{#1}}}

\makeatletter
\newcommand\given{\@ifstar{\mathrel{}\middle|\mathrel{}}{\mid}}

\DeclareRobustCommand{\abs}{\@ifstar\@abs\@@abs}
\newcommand{\@abs}[1]{\left\lvert #1 \right\rvert}
\newcommand{\@@abs}[1]{\lvert #1 \rvert}

\DeclareRobustCommand{\norm}{\@ifstar\@norm\@@norm}
\newcommand{\@norm}[1]{\left\lVert #1 \right\rVert}
\newcommand{\@@norm}[1]{\lVert #1 \rVert}
\makeatother

\newtheorem{lemma}{Lemma}
\newtheorem{theorem}{Theorem}
\newtheorem{definition}{Definition}

\newlist{assumplist}{enumerate}{1}
\setlist[assumplist]{label=(\textbf{\Alph*})}
\Crefname{assumplisti}{Assumption}{Assumptions}

\newcommand\themaintitle{Efficient and principled score estimation\\with Nystr\"om kernel exponential families}
\newcommand\therunningtitle{Efficient and principled score estimation with Nystr\"om kernel exponential families}

\begin{document}

\begin{refsection}

\renewcommand*{\thefootnote}{\fnsymbol{footnote}}%
\twocolumn[
\runningtitle{\therunningtitle}
\aistatstitle{\themaintitle}
\aistatsauthor{Danica J.\ Sutherland\footnotemark[1] \And Heiko Strathmann\footnotemark[1] \And Michael Arbel \And Arthur Gretton}
\runningauthor{Sutherland, Strathmann, Arbel, Gretton}
\aistatsaddress{%
  Gatsby Computational Neuroscience Unit, University College London\\
  \texttt{\{danica.j.sutherland,heiko.strathmann,michael.n.arbel,arthur.gretton\}@gmail.com}}
]
\renewcommand*{\thefootnote}{\arabic{footnote}}%

\begin{abstract}
  We propose a fast method with statistical guarantees for learning an exponential family density model where the natural parameter is in a reproducing kernel Hilbert space, and may be infinite-dimensional. The model is learned by fitting the derivative of the log density, the \emph{score}, thus avoiding the need to compute a normalization constant. Our approach improves the computational efficiency of an earlier solution by using a low-rank, Nystr\"om-like solution. The new solution retains the consistency and convergence rates of the full-rank solution (exactly in Fisher distance, and nearly in other distances), with guarantees on the degree of cost and storage reduction. We evaluate the method in experiments on density estimation and in the construction of an adaptive Hamiltonian Monte Carlo sampler. Compared to an existing score learning approach using a denoising autoencoder, our estimator is empirically more data-efficient when estimating the score, runs faster, and has fewer parameters (which can be tuned in a principled and interpretable way), in addition to providing statistical guarantees.
\end{abstract}

\section{INTRODUCTION}

We address the problem of efficiently estimating the natural parameter of a density in the exponential family, where this parameter may be infinite-dimensional (a member of a function space).  While finite-dimensional exponential families are a keystone of parametric statistics \citep{Brown-86}, their generalization to the fully non-parametric setting has proved challenging, despite the benefits and applications envisaged for such models  \citep{kernel-expfam}: it is difficult to construct a practical, consistent maximum likelihood solution for infinite-dimensional natural parameters \citep{Barron-91,GuQiu93,kenji:sieves}. In the absence of a tractable estimation procedure, the infinite exponential family has not seen the widespread adoption and practical successes of other nonparametric generalizations of parametric models, for instance the Gaussian and Dirichlet processes.
\renewcommand*{\thefootnote}{\fnsymbol{footnote}}%
\footnotetext[1]{These authors contributed equally.}%
\setcounter{footnote}{0}%
\renewcommand*{\thefootnote}{\arabic{footnote}}%

Recently, \citet{infdim} developed a procedure to fit infinite exponential family models to sample points drawn i.i.d.\ from a probability density, where the natural parameter is a member of a reproducing kernel Hilbert space. The approach employs a score matching procedure  \citep{hyvarinen2005estimation}, which minimizes the {\em Fisher distance}: the expected squared distance between the model {\em score} (derivative of the log density) and the score of the (unknown) true density, which can be evaluated using integration by parts. Unlike the maximum likelihood case, a Tikhonov-regularized solution can be formulated to obtain a well-posed and straightforward solution, which is a linear system defined in terms of first and second derivatives of the RKHS kernels at the sample points. Details of the model and its empirical fit are given in Section \ref{sec:backgroundExpFam}.
\citet{infdim} established consistency in Fisher, $L^r$, Hellinger, and KL distances, with rates depending on the smoothness of the density.

\citet{strathmann2015kernel_hmc} used the infinite-dimensional exponential family to approximate Hamiltonian Markov chain Monte Carlo when gradients are unavailable. In this setting, the score of the stationary distribution of the Markov chain is learned from the chain history, and used in formulating new, more efficient proposals for a Metropolis-Hastings algorithm.
Computing the full solution from \citet{infdim} has memory cost $\bigO(n^2 d^2)$ and computational cost $\bigO(n^3 d^3)$, where $n$ is the number of training samples and $d$ is the dimension of the problem; thus approximations were needed for practical implementation.
\citeauthor{strathmann2015kernel_hmc} proposed two heuristics: one using random Fourier features \citep{rahimi2007random,SutSch15,sriperumbudur2015optimal}, and the second using a finite, random set of basis points.  While these heuristics greatly improved the runtime, no convergence guarantees are known, nor how quickly to increase the complexity of these solutions with increasing $n$.

We present an efficient learning scheme for the infinite-dimensional exponential family, using a Nystr\"om approximation to the solution established in \cref{thm:nys-form}.  Our main theoretical contribution, in Theorem \ref{thm:main-big-o}, is to prove guarantees on the convergence of this algorithm for an increasing number $m$ of Nystr\"om points and $n$ training samples. Depending on the problem difficulty, convergence is attained in the regime $m\sim n^{1/3} \log n$ to $m \sim n^{1/2} \log n$, thus yielding guaranteed cost savings. The overall Fisher distance between our  solution and the true density decreases as $m,n\rightarrow \infty$  with rates that  match those of the full solution from \citet[Theorem 6]{infdim}; convergence in other distances (e.g., KL and Hellinger) either matches or is slightly worse, depending on the problem smoothness.  These tight generalization bounds  draw on recent state-of-the-art techniques developed for least-squares regression by \citet{lessismore}, which efficiently and directly control the generalization error as a function of the Nystr\"om basis, rather than relying on indirect proofs via the reconstruction error of the Gram matrix, as in e.g.\ \citet{cortes:kernel-approx}. \cref{sec:nys-expfam,sec:theory} give details.

In our experiments (Section \ref{sec:experiments}), we compare our approach against the full solution of \citet{infdim}, the heuristics of \citet{strathmann2015kernel_hmc}, and the autoencoder score estimator of \citet{alain2014regularized} (discussed in \cref{sec:dae}). We address two problem settings.
First, we evaluate score function estimation for known, multimodal densities in high dimensions.
Second, we consider adaptive Hamiltonian Monte Carlo in the style of \citet{strathmann2015kernel_hmc}, where the score is used to propose Metropolis-Hastings moves; these will be accepted more often as the quality of the learned score improves.
Our approach is more accurate, faster, and easier to tune than the autoencoder score estimate.
Moreover, our method performs as well as the full kernel exponential family solution at a much lower computational cost, and on par with previous heuristic approximations.

\section{UNNORMALIZED DENSITY AND SCORE ESTIMATION}\label{sec:backgroundExpFam}

Suppose we are given a set of points $X = \{X_b\}_{b \in [n]} \subset \R^d$
sampled i.i.d.\ from an unknown distribution with density $p_0$.
Our setting is that of \emph{unnormalized density estimation}:
we wish to fit a model $p(\cdot) = p'(\cdot) / Z(p')$ such that $p \approx p_0$ in some sense,
but without concerning ourselves with the partition function $Z(p')$,
which normalizes $p'$ such that $\int p(x) \,\ud x = 1$.
In many powerful classes of probabilistic models,
computing the partition function is intractable,
but several interesting applications do not require it,
including mode finding and sampling via Markov Chain Monte Carlo (MCMC).
This setting is closely related to that of \emph{energy-based learning} \citep{lecun:energy}.

Exponential family models with infinite-dimensional natural parameters are a particular
case for which the partition function is problematic.
Here fitting by maximum likelihood
is difficult, and becomes completely impractical in high dimensions \citep{Barron-91,GuQiu93,kenji:sieves}.

\citet{hyvarinen2005estimation} proposed an elegant approach to estimate an {\em unnormalized} density,  by minimizing the Fisher divergence,
the expected squared distance between \emph{score functions}\footnote{Here we use \emph{score} in the sense of \citet{hyvarinen2005estimation}; in traditional statistical parlance, this is the score with respect to a hypothetical location parameter of the model.} $\nabla_x \log p(x)$.
The divergence $J(p_0 \| p)$ is given by
\[
     \frac{1}{2}\int p_0(x) \norm*{\nabla_x \log p(x) - \nabla_x \log p_0(x) }_2^2 \ud x
\label{eq:fisher_score}
,\]
which under some mild regularity conditions is equal to a constant (depending only on $p_0$) plus
\[
     \int p_0(x) \sum_{i=1}^d \left[
        \partial_i^2 \log p(x)
      + \frac12 \left( \partial_i \log p(x) \right)^2
    \right]
    \ud x
\label{eq:fisher-score-parts}
.\]
We use $\partial_i f(x)$ to mean $\frac{\partial}{\partial x_i} f(x)$.
Crucially, \cref{eq:fisher-score-parts} is independent of the normalizer $Z$
and, other than the constant, depends on $p_0$ only through an expectation,
so it can be estimated by a simple Monte Carlo average.

The score function is in itself a quantity of interest, and is employed directly in several algorithms. Perhaps
best known is Hamiltonian Monte Carlo \parencite[HMC; e.g.][]{neal2011mcmc}, where the score  is used in constructing
Hamiltonian dynamics that yield fast mixing chains. Thus, if the score can be learned
from the chain history, it can be used in constructing an approximate HMC sampler with mixing properties
close to those attainable using the population score \citep{strathmann2015kernel_hmc}.  Another application
area is in constructing control functionals for Monte Carlo
integration \citep{oates:control-functionals}: again, learned score functions could be used where
closed-form expressions do not exist.

Computing unnormalized densities from a nonparametrically learned score function can be a more challenging task.
Numerical integration of the score estimate can lead to accumulating errors;
moreover, as discussed by \citet[Section~3.6]{alain2014regularized},
a given score estimate might not correspond to a valid gradient function, or might not yield a normalizable density.
The exponential family model does not suffer these drawbacks, as we will see next.

\subsection{Kernel exponential families}

We now describe the \emph{kernel exponential family} $\p$ \citep{kernel-expfam,kenji:sieves},
and how to perform unnormalized density estimation within it.
$\p$ is an infinite-dimensional exponential family:%
\begin{align}
\p=\left\{ p_f(x):=\exp\left(f(x) -A(f)\right)q_0(x)  \mid f \in \f \right\}
,
\end{align}
where $\h$ is a reproducing kernel Hilbert space \citep{berlinet},
$\f \subseteq \h$ is the set of functions for which $A(f) = \log \int \exp(f(x)) \, q_0(x) \,\ud x$, the log-partition function, is finite,
and $q_0$ is a base measure with appropriately vanishing tails.
That this is a member of the exponential family becomes apparent when we recall the reproducing property
$f(x) = \langle f, k(x,\cdot) \rangle_{\h}$: the feature map $x \mapsto k(x,\cdot)$ is the sufficient
statistic, and $f$ is the natural parameter.

Example 1 of \citet{infdim} shows how various standard finite-dimensional members of the exponential family,
including Gamma, Poisson, Binomial and so on,
fit into this framework with particular kernel functions.
When $\h$ is infinite-dimensional, $\p$ can be very rich:
for instance, when the kernel on $\R^d$ is a continuous function vanishing at infinity and integrally strictly positive definite,
then $\p$ is dense in the family of continuous densities vanishing at infinity for which $\| p/q_0 \|_{\infty}$
is bounded, with respect to the KL, TV, and Hellinger divergences \citep[Corollary 2]{infdim}.

As discussed earlier, maximum likelihood estimation is difficult due to the intractability of $A(f)$.
Instead, \citet{infdim} propose to use a score-matching approach to find an $f$ such that $p_f$ approximates $p_0$.
Their empirical estimator of \cref{eq:fisher-score-parts} is
\[
\label{eq:empirical_fisher_score}
  \hat J(f)
  = \frac{1}{n} \sum_{b=1}^n \sum_{i=1}^d
      \partial_i^2 f(X_b)
    + \frac12\left( \partial_i f(X_b) \right)^2
;
\]
this additionally drops an additive constant from \cref{eq:fisher-score-parts} that depends on $p_0$ and $q_0$ but not $f$.
The regularized loss $\hat J(f) + \frac12 \lambda \norm{f}_\h^2$ is minimized over $f \in \h$ by
\begin{align}
     f_{\lambda,n}
  &= - \frac{\hat\xi}{\lambda} + \sum_{a=1}^n \sum_{i=1}^d \beta_{(a,i)} \partial_i k(X_a, \cdot),
\label{eq:f-sum}
\\
     \hat\xi
  &= \frac1n \sum_{a=1}^n \sum_{i=1}^d \partial_i^2 k(X_a, \cdot) + \partial_i k(X_a, \cdot) \partial_i \log q_0(X_a)
\label{eq:xi-hat}
,
\end{align}
where $\beta_{(a,i)}$ denotes the $(a-1)d + i$th entry of a vector $\beta \in \R^{nd}$.
We use $\partial_i k(x, y)$ to mean $\frac{\partial}{\partial x_i} k(x, y)$,
and $\partial_{i+d} k(x, y)$ for $\frac{\partial}{\partial y_i} k(x, y)$.
To evaluate the estimated unnormalized log-density $f_{\lambda,n}$ at a point $x$,
we take a linear combination of $\partial_i k(X_a, x)$ and $\partial_i^2 k(X_a, x)$ for each sample $X_a$.
The weights $\beta$ in \eqref{eq:f-sum} are obtained by solving the $nd$-dimensional linear system
\begin{align}
(G + n \lambda I) \beta = h / \lambda  \label{eq:linear-system}
,
\end{align}
where $G\in\R^{nd\times nd}$ is the matrix collecting partial derivatives of the kernel at the training points,
$G_{(a, i), (b, j)} = \partial_i \partial_{j+d} k(X_a, X_b)$,
and $h \in \R^{nd}$ evaluates derivatives of $\hat\xi$,
$h_{(b, i)} = \partial_i \hat\xi(X_b)$.

Solving \cref{eq:linear-system} takes $\bigO(n^3 d^3)$ time and $\bigO(n^2 d^2)$ memory,
which quickly becomes infeasible as $n$ grows, especially for large $d$.
We will propose a more scalable approximation in \cref{sec:nys-expfam}.

\subsection{Fast approximate kernel regression}

The system of \cref{eq:linear-system} is related to the problem of kernel ridge regression,
which suffers from similar $\bigO(n^3)$ computational cost.
Thus we will briefly review methods for speeding up kernel regression.

\paragraph{Nystr\"om methods}
We refer here to a class of broadly related  Nystr\"om-type methods  \citep{williams-seeger,sgma,lessismore}.
The representer theorem \citep{representer} guarantees that
the minimizer of the empirical regression loss  for a training set $X = \{X_b\}_{b \in [n]}$ over the RKHS $\h$ with kernel $k$
will lie in the subspace $\h_X = \spn \{ k(X_b, \cdot) \}_{b \in [n]}$.
Nystr\"om methods find an approximate solution by optimizing over a smaller subspace $\h_Y$,
usually given by $\h_Y = \spn \{ k(y, \cdot) \}_{y \in Y}$
for a set of $m$ points $Y \subseteq X$ chosen uniformly at random.
This decreases the computational burden both of training
($\bigO(n^3)$ to $\bigO(n m^2)$ time, $\bigO(n^2)$ to $\bigO(n m)$ memory)
and testing
($\bigO(n)$ to $\bigO(m)$ time and memory).

Guarantees on the performance of Nystr\"om methods have been the topic of considerable study.
Earlier approaches have worked by
first bounding the error in a Nystr\"om approximation of the kernel matrix on the sample \citep{drineas:nys},
and then separately evaluating the impact of regression with an approximate kernel matrix \citep{cortes:kernel-approx}.
This approach, however, results in suboptimal rates;
better rates can be obtained by considering the whole problem at once \citep{elalaoui},
including its direct impact on generalization error \citep{lessismore}.

\paragraph{Random feature approximations}
Another popular method for scaling up kernel methods is to use random Fourier features \citep{rahimi2007random,SutSch15,sriperumbudur2015optimal} and their variants.
Rather than finding the best solution in a subspace of $\h$,
these methods choose a set of parametric features, often independent of the data, such that expected inner products between the features coincide with the kernel.
These methods have some attractive computational properties but generally also require the number of features to increase with the data size in a way that can be difficult to analyze: see \citet{rudi:rff} for such an analysis in regression.

\paragraph{Sketching}
 Another scheme for improving the speed of kernel ridge regression, sketching \citep{yang:sketching,woodruff:sketching} compresses the kernel matrix and the labels by multiplying with a sketching matrix.
These methods have some overlap with Nystr\"om-type approaches, and our method will encompass certain classes of sketches \citep[Appendix C.1]{lessismore}.

\subsection{Prior methods for direct score estimation} \label{sec:dae}

\citet{alain2014regularized} proposed a deep learning-based approach to directly learn a score function from samples.
Denoising autoencoders are networks trained to recover the original inputs from versions with noise added.
A denoising autoencoder  trained with $L_2$ loss and noise $\N(0, \sigma^2 I)$
can be used to construct a score estimator:
$(r_\sigma(x) - x) / \sigma^2 \approx \nabla_x \log p_0(x)$,
where $r_\sigma$ is the autoencoder's reconstruction function.
When the autoencoder has infinite capacity and reaches its global optimum,
\citet{alain2014regularized} show that this estimator is consistent as $\sigma \to 0$.
For realistic autoencoders with finite representation capacity, however, the consistency of this approach remains an open question.
Moreover, this technique has many hyperparameters to choose, both in the architecture of the network and in its trained, with no theory yet available to guide those choices.

\section{NYSTR\"OM METHODS FOR ESTIMATION IN KERNEL EXPONENTIAL FAMILIES} \label{sec:nys-expfam}
To alleviate the computational costs of the linear system in \cref{eq:linear-system}, we apply the Nystr\"om idea to the estimator of the full kernel exponential family model in \cref{eq:f-sum}.
More precisely, we select a set of $m$ ``basis'' points $Y = \{ Y_a \}_{a \in [m]}$,
and restrict the optimization in \cref{eq:f-sum} to
\[
    \h_Y := \spn \left\{ \partial_i k(Y_a, \cdot) \right\}{\substack{i \in [d] \\ a \in [m]}}
,\label{eq:h-Y}
\]
which is a subspace of $\h$ with elements that can be represented using $md$ coefficients, similar to \cref{eq:f-sum}.
Typically $Y \subset X$; in particular, $Y$ is usually chosen as a uniformly random subset of $X$.
We could, however, use any set of points $Y$ different from $X$, or even a different set of spanning vectors than $\partial_i k(Y_a, \cdot)$.

\paragraph{Dimension subsampling}
A further reduction of the computational load can be achieved by only using certain components,  $\mathcal{I}  \subset [n] \times[d]$ with $|\mathcal{I}|\leq md$, of the basis points $Y$.
Thus \eqref{eq:f-sum} is optimized over
\[
 \spn \left\{ \partial_i k(Y_a, \cdot) \mid (a,i) \in \mathcal{I}  \right\}
 .\label{eq:h-Y_dimension_subsample}
\]
In this case, each double sum over all basis points' components $\sum_{a=1}^m \sum_{i=1}^d$, such as in \cref{eq:f-sum}, would be replaced by $\sum_{(a,i)\in\mathcal{I}}$.
Our theoretical framework will support choosing whether or not to include each of the $md$ component according to user-specified probability $\rho$, such that the expected number of components $|\mathcal{I}|$ is $\rho m d$,
or of choosing exactly $\ell \le d$ components from each of $m$ points.
For the sake of notational simplicity, we will give all results in the main body for the case of using all components as in \cref{eq:h-Y}, i.e.\ $\abs{\mathcal{I}} = md$, commenting on implications for subsampling across dimensions where appropriate.
We will explore the practical impact of subsampling in the experiments.

\begin{theorem} \label{thm:nys-form}
The regularized minimizer of the empirical Fisher divergence \cref{eq:empirical_fisher_score}
over $\h_Y$ \cref{eq:h-Y} is
\begin{align}
     f_{\lambda,n}^m
  &= \sum_{a=1}^m \sum_{i=1}^d (\beta_Y)_{(a,i)} \partial_i k(Y_b, \cdot),
\label{eq:nystrom-f-sum}
\\
     \beta_Y
  &= -(\tfrac1n B_{XY}\tp B_{XY} + \lambda G_{YY})^\dagger h_Y
.
\label{eq:nystrom-linear-system}
\end{align}
Here $^\dagger$ denotes the pseudo-inverse,
and $B_{XY}\in \R^{nd \times md}, G_{YY}\in\R^{md \times md}$, $h_Y \in \R^{md}$ are given by
\begin{gather}
(B_{XY})_{(b,i),(a,j)} = \partial_i \partial_{j+d} k(X_b, Y_a) \\
(G_{YY})_{(a,i),(a',j)} = \partial_i \partial_{j+d} k(Y_a, Y_{a'}) \\
\begin{aligned}
(h_Y)_{(a,i)}
    &= \frac1n \sum_{b=1}^n \sum_{j=1}^d \partial_i \partial^2_{j+d} k(Y_a, X_b)
 \\ &\phantom{= } + \partial_i \partial_{j+d} k(Y_a, X_b) \partial_j \log q_0(X_b)
.\end{aligned}
\end{gather}
\end{theorem}
The proof, which is similar to the kernel ridge regression analogue \citep{lessismore},
is given in \cref{sec:nys-form}.
In fact, we show a slight generalization (\cref{thm:nys-form-general}),
which also applies to more general subspaces $\h_Y$.

It is worth emphasizing that in order to evaluate an estimate $f_{\lambda,n}^m$,
we  need only evaluate derivatives of the kernel between the basis points $Y$ and the test point $x$.
We no longer need $X$ at all: its full contribution is summarized in $\beta_Y$.
The same is true when subsampling across dimensions, but we need to keep all points with any used components; we will come back to this in the experiments.

When $Y \subseteq X$, the above quantities are simply block-subsampled versions of the terms in the full solution \cref{eq:linear-system}.
When using dimension subsampling with $\abs{\mathcal{I}} < md$, we subsample further within the blocks.
Note, however, that when $Y = X$ we do not exactly recover the solution \cref{eq:linear-system},
because $\hat\xi$ contains components of the form $\partial_i^2 k(X_b, \cdot) \notin \h_Y$ even when $Y=X$.

Computing the $md \times md$ matrix in \cref{eq:nystrom-linear-system} takes $\bigO(nmd^2)$ memory and $\bigO(n m^2 d^3)$ time, both linear in $n$.
Computing the pseudo-inverse
takes $\bigO(m^3 d^3)$ computation, independent of $n$.
Evaluating $f_{\lambda,n}^m$ takes $\bigO(m d)$ time, as opposed to the $\bigO(n d)$ time for $f_{\lambda,n}$.
All matrix computations can be reduced further by not using all $d$ components as in \cref{eq:h-Y_dimension_subsample}, resulting in a $|\mathcal{I}| \times |\mathcal{I}|$ matrix with $|\mathcal{I}|<md$ in \cref{eq:nystrom-linear-system}.

\paragraph{Finite and lite kernel exponential families}
\citet{strathmann2015kernel_hmc} proposed two alternative approximations to the full model of \cref{sec:backgroundExpFam},
used for efficient score learning in adaptive HMC.
Both approaches currently lack convergence guarantees.

The \emph{finite} form uses an $m$-dimensional $\h$, defined e.g.\ by random Fourier features \citep{rahimi2007random},
where \cref{eq:f-sum} can be computed directly in $\h$ in time linear in $n$.
Such parametric features limit the expressiveness of the model:
\citet{strathmann2015kernel_hmc} observed that the score estimate oscillates in regions where little or no data has been observed,
leading to poor HMC behavior when the sampler enters those regions.
We thus do not further pursue this approach in the present work.

The \emph{lite} approximation instead finds the best estimator $f \in \spn\{ k(x, \cdot) \}_{x \in X}$.
This has a similar spirit to Nystr\"om approaches,
but note the differing basis from \cref{eq:f-sum},
which is based on kernel \emph{derivatives},
and that it uses the entirety of $X$,
so the dependence on $n$ is improved only by simple subsampling.
\citet{strathmann2015kernel_hmc} derived an estimator only for Gaussian kernels.
Our generalized version of \cref{thm:nys-form} (\cref{thm:nys-form-general} in the appendix)
covers the basis used by the lite approximation,
allowing us to generalize this method to basis sets $Y \ne X$ and to kernels other than the Gaussian;
\cref{sec:lite-generalization} discusses this in more detail.

\section{THEORY} \label{sec:theory}
We analyze the performance of our estimator in the well-specified case:
assuming that the true density $p_0$ is in $\p$ (and thus corresponds to some $f_0 \in \h$),
we obtain both the parameter convergence of $f_{\lambda,n}^m$ to $f_0$
and the convergence of the corresponding density $p_{f_{\lambda,n}^m}$ to the true density $p_0$.

\begin{theorem} \label{thm:main-big-o}
  Assume the conditions listed in Appendix \ref{sec:assumptions} (similar to those of \citet{infdim} for the well-specified case),
  and use the $\h_Y$ of $\cref{eq:h-Y}$ with
  the basis set $Y$ chosen uniformly at random from the size-$m$ subsets of the training set $X$,
  and all $m d$ components included.
  Let $\beta \ge 0$ be the range-space smoothness parameter of the true density $f_0$,
  and define $b = \min\left(\beta,\, \frac12\right)$,
  $\theta = \frac{1}{2 (b+1)} \in [\frac13, \frac12]$.
  As long as $m = \Omega\left( n^\theta \log n \right)$,
  then with $\lambda = n^{-\theta}$ we obtain
  \begin{gather}
    \norm{f_{\lambda,n}^m - f_0}_\h = \bigO_{p_0}\left( n^{-\frac{b}{2(b+1)}} \right),
    \\
    J(p_0 \| p_{f_{\lambda,n}^m}) = \bigO_{p_0}\left( n^{-\frac{2b+1}{2(b+1)}} \right)
  .\end{gather}
  The first statement implies that
  $p_{f_{\lambda,n}^m}$ also converges to $p_0$
  in $L_r$ ($1 \le r \le \infty$) and Hellinger distances
  at a rate $\bigO_{p_0}\left( n^{-\frac{b}{2(b+1)}} \right)$,
  and that $\mathrm{KL}(p_0 \| p_{f_{\lambda,n}^m}), \mathrm{KL}(p_{f_{\lambda,n}^m} \| p_0)$
  are each $\bigO_{p_0}\left( n^{-\frac{b}{b+1}} \right)$.
\end{theorem}
The rate of convergence in $J$ exactly matches the rate for the full-data estimator $f_{\lambda,n}$ shown by \citet{infdim} in $J$;
the rates in other divergences essentially match, except that ours saturate slightly sooner as $\beta$ increases.
Thus, for any problem satisfying the assumptions, we can achieve the same statistical properties as the full-data setting with $m = \Omega\left( \sqrt n \log n \right)$,
while in the smoothest problems we need only $m = \bigOmega\left( n^{1/3} \log n \right)$.

This substantial reduction in computational expense is in contrast to the comparable analysis for kernel ridge regression \citep{lessismore},
which for the hardest problems requires $m = \bigOmega(n \log n)$, giving no computational savings at all.
In the best general case, it also needs $m = \bigOmega(n^{1/3} \log n)$.
This rate was  itself a significant advance:
a prior analysis based on stability of the kernel approximation \citep{cortes:kernel-approx}
results in a severe additional penalty when using Nystr\"om, matching the worst-case error rates for
the full solution, yet still requiring $m=\bigOmega(n)$ (i.e., according to the earlier reasoning, we would not be guaranteed to benefit from improved rates in easier problems).

A finite-sample version of \cref{thm:main-big-o}, with explicit constants, is shown in \cref{sec:bound-proof} (and used to prove \cref{thm:main-big-o}).
That version also includes rates for dimension subsampling.

\paragraph{Proof outline}
Each of the losses considered in \cref{thm:main-big-o} can be bounded in terms of
$\norm{f - f_0}_\h$.
We decompose this loss relative to $f_\lambda^m = \argmin_{f \in \h_Y} J(f) + \frac12 \lambda \norm{f}_\h^2$, the best regularized estimator in population with the particular basis $Y$.
That is,
\[
  \norm{f_{\lambda,n}^m - f_0}_\h
  \le \norm{f_{\lambda,n}^m - f_\lambda^m}_\h + \norm{f_\lambda^m - f_0}_\h
\label{eq:decomp-hi}
.\]

The first term on the right-hand side of \cref{eq:decomp-hi} is the \emph{estimation error},
which represents our error due to having a finite number of samples $n$:
this term decreases as $n \to \infty$,
but it will increase as $\lambda \to 0$.
It could conceivably increase as $m \to \infty$ as well,
but we show using concentration inequalities in $\h$ that no matter the $m$,
the estimation error is $\bigO_{p_0}\left( \frac{1}{\lambda \sqrt n} \right)$.

The last term of \cref{eq:decomp-hi} is the \emph{approximation error},
where ``approximation'' refers both to the regularization by $\lambda$
and the restriction to the subspace $\h_Y$.
This term is independent of $n$;
it decreases as $\h_Y$ grows (i.e.\ as $m \to \infty$),
and also with $\lambda \to 0$,
as we allow ourselves to more directly minimize the population risk.
The key to bounding this term is to exploit the nature of the space $\h_Y$.
This can be done by analogy with the treatment of the ``computational error'' term of \citet{lessismore},
where we show that any components of $f_0$ not lying within $\h_Y$
are relatively small in the parts of the space we observe; this is the only step of the proof
that depends on the specific basis $\h_Y$.
Having handled this contribution,
we show that the approximation error term is $\bigO_{p_0}\left( \lambda^b \right)$
as long as $m = \bigOmega\left( \frac1\lambda \log\frac1\lambda \right)$.

The decay of the two terms is then optimized when $\lambda = n^\theta$, with $\theta$ as given in the proof.

The rate in Fisher divergence $J$ is better because that metric
is weighted towards parts of the space where we actually see data,
as opposed to uniformly across $\h$ as in \cref{eq:decomp-hi}.
Our proof technique, similarly to that of \citet{infdim}, allows us to account for this with an improved dependence on $\lambda$ in the evaluation of both estimation and approximation errors.

\paragraph{Remarks}
Our proof uses techniques both from the analysis of the full-data estimator \citep{infdim}
and from an analysis of generalization error for Nystr\"om-subsampled kernel ridge regression \citep{lessismore}.
There are some major differences from the regression case, however.
The decomposition \eqref{eq:decomp-hi} differs from the regression decomposition (\citeauthor{lessismore}'s Appendix E),
as differences in the structure of the problem make the latter inapplicable.
Correlations between dimensions in our setup also make certain concentration results much more difficult:
compare our \cref{sec:sumCorrelatedOperators} to \citeauthor{lessismore}'s Proposition 8.

Approaches like those of \citet{elalaoui,yang:sketching},
which bound the difference in training error of Nystr\"om-type approximations to kernel ridge regression,
are insufficient for our purposes:
we need to ensure that the estimated function $f_{\lambda,n}^m$ converges to $f_0$ everywhere,
so that the full distribution matches, not just its values at the training points.
In doing so, our work is heavily indebted to \citet{cdv}, as are \citet{lessismore,infdim}.

We previously noted that using $Y = X$ does not yield an identical estimator,
$f_{\lambda,n}^n \ne f_{\lambda,n}$.
In fact, we could achieve this by additionally including $\hat\xi$ within \cref{eq:h-Y},
but since evaluating $\hat\xi$ requires touching all the data points
we would lose the test-time improvements achieved by the estimator of \cref{thm:nys-form}.
Alternatively, we could still ``forget'' points, but double the size of the basis, by including $\partial_i^2 k(X_a, \cdot)$.
In practice, $f_{\lambda,n}^n$ performs about as well as $f_{\lambda,n}$,
so neither method seems necessary.
See also \cref{remark:y=x} for more theoretical intution on why this may not be needed.

\section{EXPERIMENTS} \label{sec:experiments}
We now validate our estimator empirically.
We first consider synthetic densities in \cref{sec:toy},
where we know the true densities  and can evaluate convergence of the score estimates analytically with \cref{eq:fisher_score}, including a case with subsampled basis components in \cref{sec:exp:dim_sub_samp}.
In \cref{sec:hmc} we evaluate our estimator in the gradient-free Hamiltonian Monte Carlo setting of \citet{strathmann2015kernel_hmc}, where (in the absence of a ground truth) we compare the efficiency of the resulting sampler.

For all exponential family variants, we take $q_0$ to be a
uniform distribution with support encompassing the samples,
and use a Gaussian kernel $k(x, y) = \exp\left( - \norm{x - y}^2 / \sigma \right)$,
tuning the bandwidth $\sigma$ and regularization parameter $\lambda$ via a validation set.
We compare the following models:
\begin{description}[itemsep=0pt,topsep=0pt,leftmargin=1em]
\item[\full:]
\citet{infdim}'s model, \cref{eq:linear-system,eq:f-sum}.

\item[\lite:]
\citet{strathmann2015kernel_hmc}'s heuristic approximation,
which subsamples the dataset $X$ to size $m$, and uses the basis $\{ k(X_a, \cdot) \}$,
ignoring the remaining datapoints.
We use the regularization from their latest code,
$\lambda (\norm{f}_\h^2 + \norm{\beta}_2^2)$.

\item[\nystrom:]
The estimator of \cref{thm:nys-form}, choosing $m$ distinct data points uniformly at random for $Y$.
For numerical stability, we add $10^{-5} I$ to the matrix being inverted in \cref{eq:nystrom-linear-system}, corresponding to a small $L_2$ regularizer on the weights $\beta$.

\item[\dae:]
  The model of \citet{alain2014regularized}, where
  we train a two-layer denoising autoencoder,
  with $\tanh$ code activations and linear decoding.
  We train with decreasing noise levels
  ($100\sigma$, $10\sigma$, $\sigma$),
  using up to 1000 iterations of BFGS each.
  We tune the number of hidden units and $\sigma$;
  while \citet{alain2014regularized} recommend simply choosing some small $\sigma$,
  this plays a similar role to a bandwidth,
  and its careful choice is essential.
  We differentiate the score estimate to obtain the second derivative needed in \cref{eq:fisher-score-parts}.
\end{description}
See
\httpsurl{github.com/karlnapf/nystrom-kexpfam}
for code for the models and to reproduce the experiments.

\subsection{Score convergence on synthetic densities} \label{sec:toy}

  We first consider two synthetic densities, where the true score is available.
  The \ring{} dataset takes inspiration from the ``spiral'' dataset of \citet[Figure 5]{alain2014regularized},
  being a similarly-shaped distribution but possessing a probability density for evaluation purposes.
  We sample points uniformly along three circles with radii $(1,3,5)$ in $\R^2$ and add $\mathcal{N}(0, 0.1^2)$ noise in the radial direction.
  We then add extra dimensions consisting of independent Gaussian noise with standard deviation $0.1$.
  The \grid{} dataset is a more challenging variant of
   the 2-component mixture example of \citet[Figure 1]{infdim}.
  We fix $d$ random vertices of a $d$-dimensional hypercube;
  the target is a mixture of normal distributions, one at each vertex.

  For each run, we generate $n=500$ training points and estimate the score on $1500$ (\grid) or $5000$ (\ring) newly generated test points.
  We estimate the true score \cref{eq:fisher_score} on these test points to ensure a ``best case'' comparison of the models, though using \cref{eq:fisher-score-parts} leads to indistinguishable parameter selections and performance.
  For \lite{} and \nystrom{}, we independently evaluated the parameters for each subsampling level.
  We report performances for the best parameters found for each method.
  All experiments were conducted in a single CPU thread for timing comparisons, although multi-core parallelization is straightforward for each model.

  \cref{fig:convergence} shows convergence of the score as the dimension increases.
  On both the \ring{} and \grid{} datasets, \nystrom{} performs very close to the full solution,
  while showing  large computational savings.
  With a reasonable drop in score at $m=42$, we achieve a major reduction in cost and storage over the original $n=500$ sample size.
  The  \lite{} performance is similar to that of \nystrom{} at comparable levels of data retention.
  As expected, the performance of \nystrom{} gets closer to that of \full{} as $m$ increases towards $n$.
  The autoencoder  performs consistently worse than any of the kernel models, on both datasets.
  Autoencoder results are also strongly clustered, with only small performance improvements as the number of hidden units increases.
  As the \grid{} data reaches $20$ dimensions, all solutions start to converge to a similar score.
  None of the methods are able to learn the structure for this number of training points and dimensions; all solutions effectively revert to smooth, uninformative estimates.

  The \lite{} solution is fastest, followed by \nystrom{} for low to moderate $m$, with significant savings over the full solution even at $m=167$ on \grid{}, and across all $m$ on \ring{}.
  The additional cost of \nystrom{} over \lite{} arises since it computes all derivatives at the retained samples.
  Autoencoder runtimes are longer than the other methods, although we point out that the settings of \citet{alain2014regularized} are not optimized for runtime.
  We observed, however,
  that replacing BFGS with stochastic gradient descent
  or avoiding the decreasing noise schedule
  both lead to instabilities in the solution.

  \begin{figure}[ht]
  \centering
  \begin{subfigure}[b]{\columnwidth}
      \includegraphics[width=\textwidth]{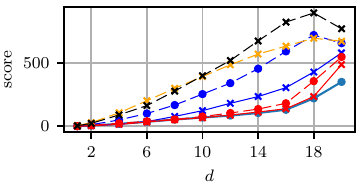}
      \includegraphics[width=\textwidth]{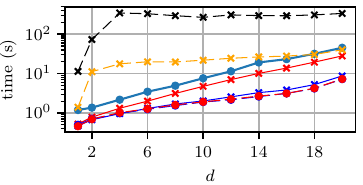}
      \includegraphics[width=.9\columnwidth]{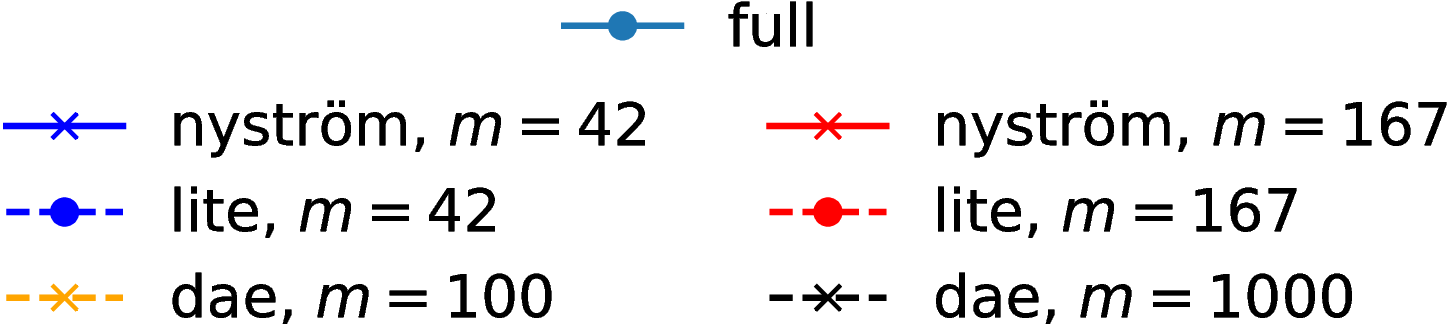}
      \caption{Scores and runtimes on the \grid{} dataset.}
  \end{subfigure}
  \begin{subfigure}[b]{\columnwidth}
      \includegraphics[width=\columnwidth]{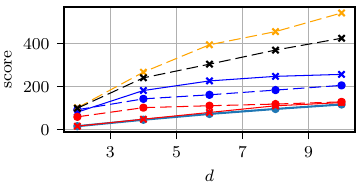}
      \caption{Scores on the \ring{} dataset; same labels as for \grid{}.}
  \end{subfigure}
  \caption{Convergence and timing on synthetic data.}
  \label{fig:convergence}
  \end{figure}

  \subsubsection{Dimension subsampling}
  \label{sec:exp:dim_sub_samp}
  To quantify the effect of subsampling components of the Nystr\"om basis in \cref{eq:h-Y_dimension_subsample}, we repeat the previous \grid{} experiment with another version of our estimator:
  \nystromd{} has the same number $md$ of basis functions as \nystrom{}, but rather than using all $d$ components of $m$ uniformly chosen training points, we uniformly choose $md$ of {\em all} available components. That is, we pick $Y=X$ in \cref{eq:h-Y_dimension_subsample} and $\mathcal{I}\subset [n] \times [d], |\mathcal{I}|=md$.
  This equalizes the cost of \cref{eq:nystrom-linear-system} for  \nystrom{} and \nystromd{}.

  \Cref{fig:dim_sub_samp} shows that distributing the used components across all training data helps slightly when $m$ is small.
  Yet this benefit comes at a cost: as mentioned in \cref{sec:nys-expfam}, \nystrom{} can discard training data not used in the basis after fitting.
  For \nystromd{}, however, we can only discard training data if \emph{no} components were chosen,
  so we must retain many more points.

  \begin{figure}
  \centering
  \includegraphics[width=\columnwidth]{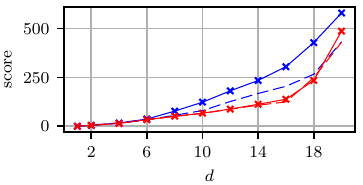}
  \\[2ex]
  \includegraphics[width=\columnwidth]{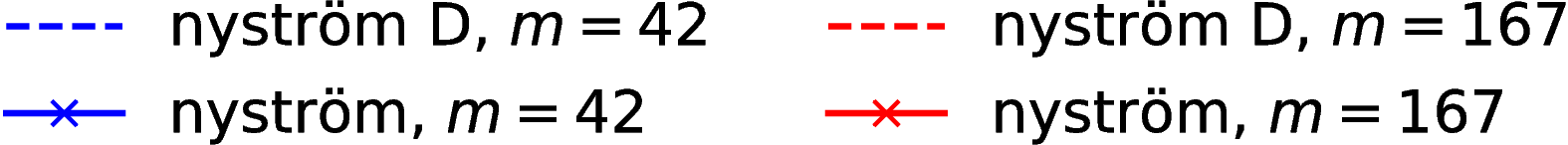}
  \caption{Dimension sub-sampling for \grid{}.}
  \label{fig:dim_sub_samp}
  \end{figure}

  \subsection{Gradient-free Hamiltonian Monte Carlo} \label{sec:hmc}

  Our final experiment follows methodology and code by \citet{sejdinovic_kernel_2014,strathmann2015kernel_hmc} in constructing a gradient-free HMC sampler using score estimates learned on the previous MCMC samples.
  Our goal is to efficiently sample from the marginal posterior over hyperparameters of a Gaussian process (GP) classifier on the UCI Glass dataset \citep{uci}.
  Closed-form expressions for the score (and therefore HMC itself) are unavailable, due to the intractability of the marginal data likelihood given the hyperparameters.
  But one can construct a Pseudo-Marginal MCMC method using an Expectation Propagation approximation to the GP posterior and importance sampling \citep{FilipponeIEEETPAMI13}.
  We compare all score estimators' ability to generate an HMC-like proposal as in \citet{strathmann2015kernel_hmc}.
  An accurate score estimate would give proposals close to an HMC move, which would have high acceptance probability.
  Thus higher acceptance rates indicate better score estimates.

  Our experiment assumes the idealized scenario where a burn-in is successfully completed.
  We run 40 random walk adaptive-Metropolis MCMC samplers for $30\,000$ iterations, discard the first $10\,000$ samples, and thin by a factor of $400$. Merging these samples results in $2\,000$ posterior samples.
  We fit all score estimators on a random subset of $n=500$ of these samples,
  and use the remaining  1500 samples to tune the model hyperparameters.
  The validation  surface obtained for \nystrom{} by the estimated score objective on the held-out set
   is shown in \cref{fig:hmc_glass}: it is smooth and easily optimized. For \dae{} (not shown here), a well-tuned level of corruption noise is essential.
  Starting from a random point of the initial posterior sketch, we construct trajectories along the surrogate Hamiltonian using $100$ steps of size $0.1$, and a standard Gaussian momentum.
  We compute the hypothetical acceptance probability for each step, and average over the trajectory.

  \Cref{fig:hmc_glass} shows the results averaged over $200$ repetitions.
  As before, \nystrom{} matches the performance of \full{} for $m = n=500$, while for $m=100$ it attains a high acceptance rate at a considerably reduced computational cost.
  It also reliably outperforms \lite{} for lower $m$,
  which might occur since \lite{} sub-samples the data while \nystrom{} only sub-samples the basis.
  \dae{} does relatively poorly, despite a large grid-search for its hyperparameters.
  For any of the models, untuned hyperparameters yield an acceptance rate close to zero.

  \begin{figure}
  \begin{subfigure}[b]{\columnwidth}
      \centering
      \includegraphics[width=.8\columnwidth]{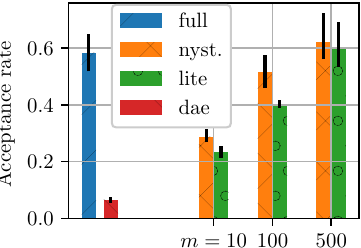}
      \caption{HMC acceptance rates, with 90\% quantiles.}
  \end{subfigure}

  \begin{subfigure}[b]{\columnwidth}
      \centering
      \includegraphics[width=.7\columnwidth,trim={0 0 0 4mm},clip]{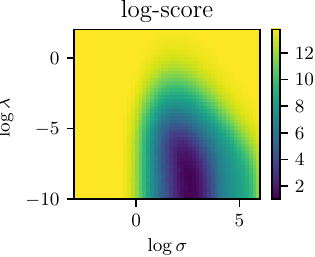}
      \caption{Log scores for various hyperparameters, for \nystrom{} with $m = 42$.}
  \end{subfigure}
  \caption{Results for GP hyperparameter optimization on the UCI Glass dataset.}
  \label{fig:hmc_glass}
  \end{figure}

  \section{CONCLUSION}
  We proposed a Nystr\"om approximation for score matching in kernel exponential families.
  \Cref{thm:main-big-o} establishes that the proposed algorithm can achieve the same or nearly the same bound on convergence as the full algorithm, with $m \ll n$.
  We also demonstrated the efficacy of the approach on challenging synthetic datasets and on an approximate HMC problem for optimizing GP hyperparameters.
  These cost reductions help make estimation in this rich family of distributions practical.

  \clearpage
  \subsubsection*{Acknowledgements}
  The authors would like to thank Mladen Kolar for productive discussions.

  \printbibliography

  \end{refsection}

\onecolumn
\begin{refsection}

\begin{center}
\ifarxiv
{\Large \textbf{Appendices}}
\else
{\LARGE \textbf{\themaintitle:\\Supplementary material}}

\fancyfoot[C]{\thepage}
\fi
\end{center}
\begin{appendices}
\crefalias{section}{appsec}
\crefalias{subsection}{appsec}
\crefalias{subsubsection}{appsec}

We now
prove \cref{thm:nys-form,thm:main-big-o}, as well as providing a finite-sample bound with explicit constants (\cref{thm:finite-bound}).

In \cref{sec:prelims}, we begin with a review of necessary notation and definitions of all necessary objects, as well as an overview of relevant theory for the full kernel exponential family estimator by \citet{infdim}.
In \cref{sec:nys-form}, we establish a representer theorem for our Nyström estimator and prove \cref{thm:nys-form}.
We address consistency and convergence in \cref{sec:bound-proof}, by first decomposing and bounding the error in \cref{sec:decomp}, then developing probabilistic inequalities in \cref{sec:prob}, and finally collecting everything into a final bound to prove \cref{thm:main-big-o} in \cref{sec:final-bound}.
\Cref{sec:aux} establishes auxiliary results used in the proofs, including tools for dimension subsampling, and in particular a concentration inequality for sums of correlated random operators in \cref{sec:sumCorrelatedOperators}.

\section{Preliminaries} \label{sec:prelims}
We will first establish some definitions that will be useful throughout,
as well as overviewing some relevant results from \citet{infdim}.

\subsection{Notation}
Our notation is mostly standard:
$\h$ is a reproducing kernel Hilbert space of functions $\Omega \subseteq \R^d \to \R$
with inner product $\langle \cdot, \cdot \rangle_\h$ and norm $\norm{\cdot}_\h$,
with a kernel $k : \Omega \times \Omega \to \R$ given by the reproducing property,
$k(x, y) = \langle k(x, \cdot), k(y, \cdot) \rangle_\h$.
The reproducing property for kernel derivatives \citep[Lemma 4.34]{steinwart} will also be important:
$\langle \partial_i k(x, \cdot), f \rangle_\h = \partial_i f(x)$
as long as $k$ is differentiable; the same holds for higher-order derivatives.

We use $\norm{\cdot}$ to denote the operator norm $\norm{A} = \sup_{f : \norm{f}_\h \le 1} \abs{\langle f, A f \rangle_\h}$,
and $A^*$ for the adjoint of an operator $A : \h_1 \to \h_2$, $\langle A f, g \rangle_{\h_2} = \langle f, A^* g \rangle_{\h_1}$.
$\lambda_{\max}(A)$ denotes the algebraically largest eigenvalue of $A$.
For elements $f \in \h_1$, $g \in \h_2$ we define $f \otimes g$ to be the tensor product,
viewed as an operator from $\h_2$ to $\h_1$ with $(f \otimes g) h = f \langle g, h \rangle_{\h_2}$;
note that $(f \otimes g)^* = g \otimes f$
and that $A (f \otimes g) B = (A f) \otimes (B^* g)$.

$C^1(\Omega)$ denotes the space of continuously differentiable functions on $\Omega$,
and $L^r(\Omega)$ the space of $r$-power Lebesgue-integrable functions.

As in the main text, $x_{(a,i)}$ will denote $x_{(a-1)d + i}$.

\subsection{Operator definitions}
The following objects will be useful in our study:
$C$, $\xi$, and their estimators were defined by \citet{infdim}.
$C$ is similar to the standard covariance operator in similar analyses \citep{cdv,lessismore}.

\begin{definition}
Suppose we have a sample set $X = \{ X_a \}_{a \in [n]} \subset \R^d$.
For any $\lambda > 0$, define the following:
\begin{align}
     C
  &= \E_{x \sim p_0}\left[ \sum_{i=1}^d \partial_i k(x, \cdot) \otimes \partial_i k(x, \cdot) \right]
     : \h \to \h
;\qquad C_\lambda = C + \lambda I
\label{eq:c}
\\   \xi
  &= - C f_0 = \E_{x \sim p_0}\left[ \sum_{i=1}^d \partial_i k(x, \cdot) \partial_i \log q_0(x) + \partial_i^2 k(x, \cdot) \right]
     \in \h
\label{eq:xi}
\\    Z_X
  &= \sum_{b=1}^n \sum_{i=1}^d e_{(b,i)} \otimes \partial_i k(X_b, \cdot)
     : \h \to \R^{nd}
\label{eq:z-x}
;\end{align}
    here $e_{(b,i)} \in \R^{nd}$ has component $(b-1)d + i$ equal to 1 and all others 0.

Define estimators of \cref{eq:c,eq:xi} by
\begin{align}
     \hat C
  &= \frac1n Z_X^* Z_X
   = \frac1n \sum_{a=1}^n \sum_{i=1}^d \partial_i k(X_a, \cdot) \otimes \partial_i k(X_a, \cdot)
   : \h \to \h
\label{eq:c-hat}
;\qquad
     \hat C_\lambda
   = \hat C + \lambda I
\\   \hat\xi
  &= \frac1n \sum_{a=1}^n \sum_{i=1}^d \partial_i k(X_a, \cdot) \partial_i \log q_0(X_a) + \partial_i^2 k(X_a, \cdot)
   \in \h
\label{eq:xi-hat-app}
.\end{align}

Further define:
\begin{align}
\N_\infty(\lambda)
&:= \sup_{x \in \Omega} \sum_{i=1}^d \norm*{C_\lambda^{-\frac12} \partial_i k(x, \cdot)}_\h^2
\\
\N'_\infty(\lambda)
&:= \sup_{\substack{x \in \Omega \\ i \in [d]}} \norm*{C_\lambda^{-\frac12} \partial_i k(x, \cdot)}_\h^2
.\end{align}

\end{definition}

Here, $Z_X$ evaluates derivatives of its input at the points of $X$, $(Z_X f)_{(b,i)} = \partial_i f(X_b)$,
whereas $Z_X^*$ constructs linear combinations: for $\alpha \in \R^{nd}$,
$Z_X^* \alpha = \sum_{b=1}^n \sum_{i=1}^d \alpha_{(b,i)} \partial_i k(X_b, \cdot)$.

\subsection{Assumptions}\label{sec:assumptions}
We will need the following assumptions on $p_0$, $q_0$, and $\h$:
\begin{assumplist}
  \item \label{assump:well-spec}
    (Well-specified)
    The true density is $p_0 = p_{f_0} \in \p$, for some $f_0 \in \f$.
  \item \label{assump:aapo-1}
    $\supp p_0 = \Omega$ is a non-empty open subset of $\R^d$,
    with a piecewise smooth boundary $\partial\Omega := \bar{\Omega} \setminus \Omega$,
    where $\bar\Omega$ denotes the closure of $\Omega$.
  \item \label{assump:aapo-2}
    $p_0$ is continuously extensible to $\bar\Omega$.
    $k$ is twice continuously differentiable on $\Omega \times \Omega$,
    with $\partial^{\alpha,\alpha} k$ continuously extensible to $\bar\Omega \times \bar\Omega$ for $\abs{\alpha} \le 2$.
  \item \label{assump:aapo-3}
    $\partial_i \partial_{i+d} k(x, x') \rvert_{x' = x} p_0(x) = 0$ for $x \in \partial\Omega$,
    and for all sequences of $x \in \Omega$ with $\norm{x}_2 \to \infty$
    we have have
    $\left. p_0(x) \sqrt{\partial_i \partial_{i+d} k(x, x')} \right\rvert_{x' = x} = \littleO\left( \norm{x}^{1-d} \right)$
    for each $i \in [d]$.
  \item (Integrability) \label{assump:1-integrable}
    For all $i \in [d]$, each of
    \[
      \left. \partial_i \partial_{i+d} k(x, x') \right\rvert_{x' = x},
      \left. \sqrt{\partial_i^2 \partial_{i+d}^2 k(x, x')} \right\rvert_{x' = x},
      \left. \partial_i \log q_0(x) \sqrt{\partial_i^2 \partial_{i+d}^2 k(x, x')} \right\rvert_{x' = x}
      \label{eq:1-integrable}
    \]
      are in $L^1(\Omega, p_0)$. Moreover, $q_0 \in C^1(\Omega)$.
  \item (Range space) \label{assump:range-space}
    $f_0 \in \range(C^\beta)$ for some $\beta \ge 0$,
    and $\norm*{C^{-\beta} f_0}_\h < R$ for some $R < \infty$.
    The operator $C$ is defined by \cref{eq:c}.
  \item (Bounded derivatives) \label{assump:bounded}
    $\supp(q_0) = \h$,
    and the following quantities are finite:
    \begin{gather}
       \kappa_1^2 := \sup_{\substack{x \in \Omega \\ i \in [d]}}
          \left. \partial_i \partial_{i+d} k(x, x')  \right\rvert_{x' = x}
,\ \kappa_2^2 := \sup_{\substack{x \in \Omega \\ i \in [d]}}
          \left. \partial_i^2 \partial_{i+d}^2 k(x, x') \right\rvert_{x' = x}
,\ Q := \sup_{\substack{x \in \Omega \\ i \in [d]}}
          \abs*{\partial_i \log q_0(x)}
    .\end{gather}
  \item (Bounded kernel) \label{assump:bounded-kernel}
    $\kappa^2 := \sup_{x \in \Omega} k(x, x)$ is finite.
\end{assumplist}
These assumptions, or closely related ones, were all used by \citet{infdim} for various parts of their analysis.
\cref{assump:aapo-1,assump:aapo-2,assump:aapo-3} ensure that the form for $J(p_0 \| p)$ in \cref{eq:fisher-score-parts} is valid.
\cref{assump:1-integrable} implies $J(p_0 \| p_f)$ is finite for any $p_f \in \p$.
\cref{assump:bounded} is used to get probabilistic bounds on the convergence of the estimators, and implies \cref{assump:1-integrable}. Note that $\kappa_2^2 < \infty$ and $Q < \infty$ can be replaced by $L^2(\Omega, p_0)$ integrability assumptions as in \citet{infdim} without affecting the asymptotic rates, but $\kappa_1^2 < \infty$ is used to get Nyström-like rates.
\cref{assump:bounded-kernel} is additionally needed for the convergence in $L^r$, Hellinger, and KL distances.

Note that under \ref{assump:bounded},
$
  \N_\infty(\lambda)
    \le d \N'_\infty(\lambda)
    \le \frac{d \kappa_1^2}{\lambda}
,$
and $\norm C \le d \kappa_1^2$.

\subsection{Full-data result}
This result is essentially Theorem 3 of \citet{infdim}.
\begin{lemma} \label{thm:full-data-form}
Under \cref{assump:well-spec,assump:aapo-1,assump:aapo-2,assump:aapo-3,assump:1-integrable},
\begin{align}
     J(f)
  &= J(p_0 \| p_f)
   = \frac12 \langle f - f_0, C (f - f_0) \rangle_\h
   = \frac12 \langle f, C f \rangle_\h + \langle f, \xi \rangle_\h + J(p_0 \| q_0)
.\end{align}
Thus for $\lambda > 0$,
the unique minimizer of the regularized loss function
$J_\lambda(f) = J(f) + \frac12 \lambda \norm{f}_\h^2$
is
\[
  f_\lambda
  = \argmin_{f \in \h} J_\lambda(f)
  = - C_\lambda^{-1} \xi
  = C_\lambda^{-1} C f_0
.\]

Using the estimators \cref{eq:c-hat,eq:xi-hat-app},
define an empirical estimator of the loss function \cref{eq:empirical_fisher_score}, up to the additive constant $J(p_0 \| q_0)$, as
\[
  \hat J(f) = \frac12 \langle f, \hat C f \rangle_\h + \langle f, \hat\xi \rangle_\h
.\]
There is a unique minimizer of
$\hat J_\lambda(f) = \hat J(f) + \frac12\lambda \norm{f}_\h^2$:
\[
  f_{\lambda,n}^m = \argmin_{f \in \h} \hat J_\lambda(f) = - \hat C_\lambda^{-1} \hat\xi
.\]
\end{lemma}

$f_{\lambda,n}^m$ can be computed according to Theorem 4 of \citet{infdim},
using \cref{eq:f-sum,eq:linear-system}.

\subsection{Subsampling}

In our Nyström projections,
we will consider a more general $\h_Y$ than \cref{eq:h-Y},
allowing any finite-dimensional subspace of $\h$.

\begin{definition}[Subsampling operators]
Let $Y = \{ y_a \}_{a \in [m]} \subset \h$ be some basis set,
and let its span be $\h_Y = \spn(Y)$;
note that \cref{eq:h-Y} uses $y_{(a,i)} = \partial_i k(Y_a, \cdot)$.
Then define
\[
      Z_Y
   = \sum_{a=1}^m e_{a} \otimes y_a
     : \h \to \R^{m}
\label{eq:z-y}
;\]
let $Z_Y$ have singular value decomposition $Z_Y = U \Sigma V^*$,
where $\Sigma \in \R^{t \times t}$ for some $t \le M$.
Note that $V V^* = P_Y$ is the orthogonal projection operator onto $\h_Y$,
while $V^* V$ is the identity on $\R^t$.

For an operator $A : \h \to \h$, let
\[
  g_Y(A) = V (V^* A V)^{-1} V^*
  \label{eq:g-y}
.\]
\end{definition}

The projected inverse function $g_Y$, defined by \citet{lessismore}, will be crucial in our study,
and so we first establish some useful properties of it.
\begin{lemma}[Properties of $g_Y$]\label{thm:gy-props}
Let $A : \h \to \h$ be a positive operator,
and define $A_\lambda = A + \lambda I$ for any $\lambda > 0$.
The operator $g_Y$ of \cref{eq:g-y} satisfies the following:
\begin{enumerate}[label=(\roman*)]
  \item $g_Y(A) P_Y = g_Y(A)$,
  \item $P_Y g_Y(A) = g_Y(A)$,
  \item $g_Y(A_\lambda) A_\lambda P_Y = P_Y$, \label{part:gy-invish}
  \item $g_Y(A_\lambda) = (P_Y A P_Y + \lambda I)^{-1} P_Y$, and  \label{part:gy-p}
  \item $\norm{ A_\lambda^\frac12 g_Y(A_\lambda) A_\lambda^\frac12 } \le 1$. \label{part:gy-norm}
\end{enumerate}
\end{lemma}
\begin{proof}
(i) and (ii) follow from $V^* P_Y = V^* V V^* = V^*$
and $P_Y V = V V^* V = V$, respectively.
(iii) is similar:
$g_Y(A_\lambda) A_\lambda P_Y = V (V^* A_\lambda V)^{-1} V^* A_\lambda V V^* = V V^*$.
For (iv),
\[
  P_Y
  = V V^*
  = V (V^* A_\lambda V) (V^* A_\lambda V)^{-1} V^*
  = V (V^* A_\lambda V) V^* V (V^* A_\lambda V)^{-1} V^*
.\]
But $V (V^* A_\lambda V) V^* = V (V^* A V + \lambda V^* V) V^* = (P_Y A P_Y + \lambda I) P_Y$,
so we have
\[
  P_Y = (P_Y A P_Y + \lambda I) P_Y g_Y(A_\lambda)
;\]
left-multiplying both sides by $(P_Y A P_Y + \lambda I)^{-1}$
and using (ii) yields the desired result.
Finally,
\begin{align}
     \left(A_\lambda^\frac12 g_Y(A_\lambda) A_\lambda^\frac12 \right)^2
  &= A_\lambda^\frac12 g_Y(A_\lambda) A_\lambda g_Y(A_\lambda) A_\lambda^\frac12
\\&= A_\lambda^\frac12 V (V^* A_\lambda V)^{-1} V^* A_\lambda V (V^* A_\lambda V)^{-1} V^* A_\lambda^\frac12
\\&= A_\lambda^\frac12 V (V^* A_\lambda V)^{-1} V^* A_\lambda^\frac12
\\&= A_\lambda^\frac12 g_Y(A_\lambda) A_\lambda^\frac12
,\end{align}
so that $A_\lambda^\frac12 g_Y(A_\lambda) A_\lambda^\frac12$ is a projection.
Thus its operator norm is either 0 or 1, and (v) follows.
\end{proof}

\section{Representer theorem for Nyström optimization problem (Theorem \ref{thm:nys-form})}
\label{sec:nys-form}

We will first establish some representations for $f_{\lambda,n}^m$ in terms of operators on $\h$
(in \cref{thm:nys-forms-h}),
and then show \cref{thm:nys-form-general}, which generalizes \cref{thm:nys-form}.
This parallels Appendix C of \citet{lessismore}.

\begin{lemma}\label{thm:nys-forms-h}
  Under \cref{assump:well-spec,assump:aapo-1,assump:aapo-2,assump:aapo-3,assump:1-integrable},
  the unique minimizer of $\hat J(f) + \lambda \norm{f}_\h^2$ in $\h_Y$ is
  \[
    f_{\lambda,n}^m
    = - (P_Y \hat C P_Y + \lambda I)^{-1} P_Y \hat \xi
    = - g_Y(\hat C_\lambda) \hat\xi
    \label{eq:nys-forms}
  .\]
\end{lemma}
\begin{proof}
We begin by rewriting the minimization using \cref{thm:full-data-form} as
\begin{align}
     f_{\lambda,n}^m
  &= \argmin_{f \in \h_Y} \hat J_\lambda(f)
\\&= \argmin_{f \in \h_Y}
          \frac12 \langle f, \hat C f \rangle_\h
        + \langle f, \hat\xi \rangle_\h
        + \frac12 \lambda \norm{f}_\h^2
\\&= \argmin_{f \in \h_Y}
          \frac12 \langle P_Y f, \hat C P_Y f \rangle_\h
        + \langle P_Y f, \hat\xi \rangle_\h
        + \frac12 \lambda \norm{f}_\h^2
\\&= \argmin_{f \in \h_Y} \frac12
          \left\langle \frac{1}{\sqrt n} Z_X P_Y f, \frac{1}{\sqrt n} Z_X P_Y f \right\rangle_\h
        + \langle f, P_Y \hat\xi \rangle_\h
        + \frac12 \lambda \norm{f}_\h^2
\\&= \argmin_{f \in \h_Y} \frac12
          \norm*{\frac{1}{\sqrt n} Z_X P_Y f}_\h^2
        + \lambda \left\langle f, \frac1\lambda P_Y \hat\xi \right\rangle_\h
        + \frac12 \lambda \norm{f}_\h^2
        + \frac12 \lambda \norm*{\frac1\lambda P_Y \hat\xi}_\h^2
\\&= \argmin_{f \in \h_Y} \frac12
          \norm*{\frac{1}{\sqrt n} Z_X P_Y f}_\h^2
        + \frac12 \lambda \norm*{ f + \frac1\lambda P_Y \hat\xi}_\h^2
.\end{align}
This problem is strictly convex and coercive, thus a unique $f_{\lambda,n}^m$ exists.
Now, for any $f \in \h$, we have
\[
    \norm*{f + \frac1\lambda P_Y \hat\xi}_\h^2
  = \norm*{P_Y f + \frac1\lambda P_Y \hat\xi}_\h^2
  + \norm*{(I - P_Y) f}_\h^2
,\]
so that the problem
\[
     \argmin_{f \in \h} \frac12
          \norm*{\frac{1}{\sqrt n} Z_X P_Y f}_\h^2
        + \frac12 \lambda \norm*{ f + \frac1\lambda P_Y \hat\xi}_\h^2
\]
will yield a solution in $\h_Y$.
This problem is also strictly convex and coercive,
so its unique solution must be $f_{\lambda,n}^m$.
By differentiating the objective, we can then see that
\begin{gather}
  \tfrac1n P_Y Z_X^* Z_X f_{\lambda,n}^m + \lambda f_{\lambda,n}^m + P_Y \hat \xi = 0
\\\left( P_Y \hat C P_Y + \lambda I \right) f_{\lambda,n}^m = - P_Y \hat\xi
,\end{gather}
which since $\hat C$ is positive yields the first equality of \cref{eq:nys-forms}.
The second follows from \cref{thm:gy-props} \ref{part:gy-p}.
\end{proof}

\begin{lemma}[Generalization of \cref{thm:nys-form}] \label{thm:nys-form-general}
  Under \cref{assump:well-spec,assump:aapo-1,assump:aapo-2,assump:aapo-3,assump:1-integrable},
  $f_{\lambda,n}^m$ can be computed as
\begin{align}
     f_{\lambda,n}^m
  &= Z_Y^* \beta_Y
   = \sum_{a=1}^m (\beta_Y)_{a} y_a
\\
     \beta_Y
  &= -(\tfrac1n B_{XY}\tp B_{XY} + \lambda G_{YY})^\dagger h_Y
,
\label{eq:nystrom-linear-system-general}
\end{align}
where $B_{XY}\in \R^{nd \times m}, G_{YY}\in\R^{m \times m}$, $h_Y \in \R^{m}$ are given by
\begin{align}
(B_{XY})_{(b,i),a} &= \langle \partial_i k(X_b, \cdot), y_a \rangle_\h
\label{eq:bxy}
\\
(G_{YY})_{a,a'} &= \langle y_a, y_{a'} \rangle_\h
\label{eq:gyy}
\\
(h_Y)_{a} &= \langle \hat\xi, y_a \rangle_\h
\label{eq:hy}
.
\end{align}
\end{lemma}
\begin{proof}
First, $B_{XY} = Z_X Z_Y^*$,
$G_{YY} = Z_Y Z_Y^*$,
and $h_Y = Z_Y \hat\xi$.
For example, \cref{eq:bxy} agrees with
\begin{align}
     Z_X Z_Y^*
  &= \left[ \sum_{b=1}^n \sum_{i=1}^d e_{(b,i)} \otimes \partial_i k(X_b, \cdot) \right]
     \left[ \sum_{a=1}^m y_a \otimes e_{a} \right]
\\&= \sum_{b=1}^n \sum_{i=1}^d \sum_{a=1}^m
     \langle \partial_i k(X_b, \cdot), y_a \rangle_\h
     \left[ e_{(b,i)} \otimes e_{a} \right]
.\end{align}

Recall the full-rank factorization of pseudo-inverses:
if a matrix $A$ of rank $r$ can be written as $A = FG$ for $F$, $G$ each of rank $r$,
then $A^\dagger = G^\dagger F^\dagger$ \citep[chap.\ 1, sec.\ 6, ex.\ 17]{gen-invs}.

Now we can show that the claimed form \cref{eq:nystrom-linear-system-general} matches $f_{\lambda,n}^m$ from \cref{eq:nys-forms}:
\begin{align}
     - Z_Y^* \left( \tfrac1n B_{XY}\tp B_{XY} + \lambda G_{YY} \right)^\dagger h_Y
  &= - Z_Y^* \left( \tfrac1n Z_Y Z_X^* Z_X Z_Y^* + \lambda Z_Y Z_Y^* \right)^\dagger Z_Y \hat\xi
\\&= - Z_Y^* \left( Z_Y \hat C_\lambda Z_Y^* \right)^\dagger Z_Y \hat\xi
\\&= - V \Sigma U^* \left( (U \Sigma) (V^* \hat C_\lambda V) \Sigma U^* \right)^\dagger U \Sigma V^* \hat\xi
\\&= - V \Sigma U^* (\Sigma U^*)^\dagger ( V^* \hat C_\lambda V )^\dagger (U \Sigma)^\dagger U \Sigma V^* \hat\xi
\\&= - V \Sigma U^* U \Sigma^{-1} ( V^* \hat C_\lambda V )^{-1} \Sigma^{-1} U^* U \Sigma V^* \hat\xi
\\&= - V ( V^* \hat C_\lambda V )^{-1} V^* \hat\xi
\\&= - g_Y(\hat C_\lambda) \hat\xi
   = f_{\lambda,n}^m
.\qedhere\end{align}
\end{proof}

\Cref{thm:nys-form} is the specialization of \cref{thm:nys-form-general} to $y_{(a,i)} = \partial_i k(Y_a, \cdot)$.

\subsection{Relationship to ``lite'' kernel exponential families} \label{sec:lite-generalization}
The lite kernel exponential family of \citet{strathmann2015kernel_hmc}
obtains a solution in $\h'_Y = \spn\{ k(y, \cdot) \}_{y \in Y}$,
where in that paper it was assumed that
$Y = X$,
$k(x, y) = \exp\left( - \tau^{-1} \norm{x - y}^2 \right)$,
and $q_0$ was uniform.
Their estimator, given by their Proposition 1, is
\begin{gather}
     \alpha
   = - \frac{\tau}{2} (A + \lambda I)^{-1} b
\label{eq:lite-theirs}
\\   A
   = \sum_{i=1}^d -[D_{x_i} K - K D_{x_i}]^2
\qquad
     b
   = \sum_{i=1}^d \left( \frac2\tau (K s_i + D_{s_i} K \mathbf{1} - 2 D_{x_i} K x_i) - K \mathbf{1} \right)
\end{gather}
where $x_i = \begin{bmatrix} X_{1i} & \dots & X_{ni} \end{bmatrix}\tp$,
$s_i = x_i \odot x_i$ with $\odot$ the elementwise product,
$D_x = \diag(x)$,
and $K \in \R^{m \times m}$ has entries $K_{aa'} = k(X_a, X_{a'})$.

\cref{thm:nys-form-general} allows us to optimize over $\h'_Y$;
we need not restrict ourselves to $Y = X$, uniform $q_0$, or a Gaussian kernel.
Here $y_a = k(Y_a, \cdot)$, and we obtain
\[
  \beta_Y' = - \left(\frac1n (B'_{XY})\tp B'_{XY} + \lambda G'_{YY} \right)^\dagger h'_Y
.\]
Using that for the Gaussian kernel $k$
\[
  \partial_{i} k(x, y)
  = -\frac2\tau (x_i - y_i) k(x, y)
\qquad
  \partial_{i+d}^2 k(x, y)
  = \frac{2}{\tau} \left[ \frac2\tau (x_i - y_i)^2 - 1 \right] k(x, y)
,\]
we can obtain with some algebra similar to the proof of \citet{strathmann2015kernel_hmc}'s Proposition 1 that when $Y = X$ and $q_0$ is uniform,
\[
  h'_X = \frac{2}{n \tau} b
  \qquad
  (B'_{XX})\tp B'_{XX} = \frac{4}{\tau^2} A
  \qquad
  G'_{XX} = K
.\]
Thus
\begin{align}
     \beta_X'
  &= - \left( \frac{4}{n \tau^2} A + \lambda K \right)^\dagger \frac{2}{n \tau} b
   = - \frac{\tau}{2} \left( A + \frac14 n \tau^2 \lambda K \right)^\dagger b
\label{eq:lite-ours}
.\end{align}
\cref{eq:lite-ours} resembles \cref{eq:lite-theirs},
except that our approach regularizes $A$ with $\frac14 n \tau^2 \lambda K$
rather than $\lambda I$.
This is because,
despite claims by \citet{strathmann2015kernel_hmc} in both the statement and the proof of their Proposition 1 that they minimize
$\hat J(f) + \lambda \norm{f}_\h^2$,
they in fact minimize
$\hat J(f) + \frac12 n \tau^2 \lambda \norm{\alpha}_2^2$.
Our solutions otherwise agree.

\section{Consistency and convergence rate of the estimator (Theorem \ref{thm:main-big-o})}
\label{sec:bound-proof}

To prove the consistency and convergence of $f_{\lambda,n}^m$,
we will first bound the difference between $f_{\lambda,n}^m$ in terms of various quantities (\cref{sec:decomp}),
which we will then study individually in \cref{sec:prob}
to yield the final result in \cref{sec:final-bound}.
\Cref{sec:aux} gives auxiliary results used along the way.

\subsection{Decomposition} \label{sec:decomp}

We care both about the parameter convergence
$\norm{f_{\lambda,n}^m - f_0}_\h$
and the convergence of $p_{\lambda,n}^m = p_{f_{\lambda,n}^m}$ to $p_0$ in various distances.
But by \cref{thm:full-data-form},
we know that
$J(p_0 \| p_{\lambda,n}^m) = \frac12 \norm*{C^\frac12 (f_{\lambda,n}^m - f_0)}_\h^2$.
\Cref{thm:dist-distances} additionally shows that the $L^r$, KL, and Hellinger distances between the distributions can be bounded in terms of $\norm{f_{\lambda,n}^m - f_0}_\h$.
Thus it suffices to bound
$\norm{C^\alpha (f_{\lambda,n}^m - f_0)}_\h$
for $\alpha \ge 0$.

\begin{lemma} \label{thm:decomp}
  Under \cref{assump:range-space,assump:aapo-1,assump:aapo-2,assump:aapo-3,assump:well-spec,assump:1-integrable},
  let $\alpha \ge 0$
  and define
  \[
    c(a) := \lambda^{\min\left( 0,\ a - \frac12 \right)}
            \norm{C}^{\max\left( 0,\ a - \frac12 \right)}
    ,\qquad
    \mathcal C_Y := \norm{C_\lambda^\frac12 (I - V V^*)}^2
  .\]
  Then
  \begin{multline}
    \norm{C^\alpha (f_{\lambda,n}^m - f_0)}_\h
    \le R \left( 2 \mathcal C_Y + \lambda \right) c(\alpha) c(\beta)
\\    + \frac{1}{\sqrt\lambda} \norm*{C^\alpha \hat C_\lambda^{-\frac12}}
        \Big(
          \norm{\hat\xi - \xi}_\h
        + \norm{\hat C - C} R \left(
            \left( \frac{2 \mathcal C_Y}{\sqrt\lambda} + \sqrt\lambda \right) c(\beta)
          + \norm{C}^\beta
          \right)
       \Big)
  .\end{multline}
\end{lemma}
\begin{proof}
We will decompose the error with respect to the best estimator for a fixed basis:
\begin{align}
  f_\lambda^m
  &:= \argmin_{f \in \h_Y} \frac12 \langle f, P_Y C P_Y f \rangle_\h + \langle f, P_Y \xi \rangle_\h + \frac12 \lambda \norm{f}_\h^2
\\&= - (P_Y C P_Y + \lambda I)^{-1} P_y \xi
   = - g_Y(C_\lambda) \xi
   = g_Y(C_\lambda) C f_0
.\end{align}
Then we have
\[
    \norm{C^\alpha (f_{\lambda,n}^m - f_0)}_\h
    \le \norm{C^\alpha (f_{\lambda,n}^m - f_\lambda^m)}_\h
      + \norm{C^\alpha (f_{\lambda}^m - f_0)}_\h
\label{eq:decomp-init}
.\]
We'll tackle the second term first.

\paragraph{Approximation error}
This term covers both approximation due to the basis $\h_Y$ and the bias due to regularization.
We'll break it down using some ideas from the proof of \citet{lessismore}'s Theorem 2:
\begin{align}
     f_0 - f_\lambda^m
  &= (I - g_Y(C_\lambda) C) f_0
\\&= \left( I - g_Y(C_\lambda) C_\lambda + \lambda g_Y(C_\lambda) \right) f_0
\\&= \left( I - g_Y(C_\lambda) C_\lambda (V V^*) - g_Y(C_\lambda) C_\lambda (I - V V^*) + \lambda g_Y(C_\lambda) \right) f_0
\\&= \left( (I - V V^*) - g_Y(C_\lambda) C_\lambda (I - V V^*) + \lambda g_Y(C_\lambda) \right) f_0
,\end{align}
where in the last line we used \cref{thm:gy-props} \ref{part:gy-invish}.
Thus,
using \cref{assump:range-space} and \cref{thm:gy-props} \ref{part:gy-norm},
\begin{align}
       \norm*{C^\alpha (f_\lambda^m - f_0)}_\h
  &\le \norm*{C^\alpha (I - V V^*) f_0}_\h
     + \norm*{C^\alpha g_Y(C_\lambda) C_\lambda (I - V V^*) f_0}_\h
     + \lambda \norm*{C^\alpha g_Y(C_\lambda) f_0}_\h
\\&\le \underbrace{\norm*{C^\alpha C_\lambda^{-\frac12}}}_{\mathcal S_\alpha}
       \norm*{C_\lambda^\frac12 (I - V V^*) C^\beta}
       \underbrace{\norm*{C^{-\beta} f_0}_\h}_{\le R}
\\&\quad
     + \underbrace{\norm*{C^\alpha C_\lambda^{-\frac12}}}_{\mathcal S_\alpha}
       \underbrace{\norm*{C_\lambda^\frac12 g_Y(C_\lambda) C_\lambda^\frac12}}_{\le 1}
       \norm*{C_\lambda^\frac12 (I - V V^*) C^\beta}
       \underbrace{\norm*{C^{-\beta} f_0}_\h}_{\le R}
\\&\quad
     + \lambda
       \underbrace{\norm*{C^\alpha C_\lambda^{-\frac12}}}_{\mathcal S_\alpha}
       \underbrace{\norm*{C_\lambda^\frac12 g_Y(C_\lambda) C_\lambda^\frac12}}_{\le 1}
       \underbrace{\norm*{C_\lambda^{-\frac12} C^\beta}_\h}_{\mathcal S_\beta}
       \underbrace{\norm*{C^{-\beta} f_0}_\h}_{\le R}
.\end{align}

Because $(I - V V^*)$ is a projection, we have
\[
       \norm*{C_\lambda^\frac12 (I - V V^*) C^\beta}
   \le \norm*{C_\lambda^\frac12 (I - V V^*)^2 C_\lambda^\frac12} \norm*{C_\lambda^{-\frac12} C^\beta}
   \le \norm*{C_\lambda^\frac12 (I - V V^*)}^2 \mathcal S_\beta
.\]

We can also bound the terms $\mathcal S_a$ as follows.
When $a \ge \frac12$,
the function $x \mapsto x^a / \sqrt{x + \lambda}$ is increasing on $[0, \infty)$,
so that
\[
  \mathcal S_a
  = \norm*{C_\lambda^{-\frac12} C^a}_\h
  = \frac{\norm{C}^a}{\sqrt{\norm{C} + \lambda}}
  \le \norm{C}^{a - \frac12}
.\]
When instead $0 \le a < \frac12$,
we have that
\[
  \mathcal S_a
  = \norm*{C_\lambda^{-\frac12} C^a}_\h
  \le \max_{x \ge 0} \frac{x^a}{\sqrt{x + \lambda}}
  = \sqrt{2} a^a \left(\tfrac12 - a\right)^{\frac12 - a} \lambda^{a - \frac12}
  \le \lambda^{a - \frac12}
.\]
Combining the two yields
\[
  \mathcal S_a
  \le \lambda^{\min\left( 0,\ a - \frac12 \right)}
      \norm{C}^{\max\left( 0,\ a - \frac12 \right)}
  = c(a)
,\]
and so
\begin{align}
       \norm*{C^\alpha (f_\lambda^m - f_0)}_\h
  &\le R
      \left( 2 \norm*{C_\lambda^\frac12 (I - V V^*)}^2 + \lambda \right)
      c(\alpha) c(\beta)
\label{eq:decomp-proj}
.\end{align}

\paragraph{Estimation error}
Let $D = P_Y C P_Y$, $\hat D = P_Y \hat C P_Y$.
Then
\[
  f_\lambda^m
   = - (D + \lambda I)^{-1} P_Y \xi
   = -\frac1\lambda (D + \lambda I - D) (D + \lambda I)^{-1} P_Y \xi
   = -\frac1\lambda (P_Y \xi + D f_\lambda^m)
,\]
and so the error due to finite $n$ is
\begin{align}
       f_{\lambda}^m - f_{\lambda,n}^m
  &  = (\hat D + \lambda I)^{-1} P_Y \hat\xi + f_\lambda^m
\\&  = (\hat D + \lambda I)^{-1} \left( P_Y \hat\xi + (\hat D + \lambda I) f_\lambda^m \right)
\\&  = (\hat D + \lambda I)^{-1} \left( P_Y \hat\xi + \hat D f_\lambda^m + \lambda f_\lambda^m \right)
\\&  = (\hat D + \lambda I)^{-1} \left( P_Y \hat\xi + \hat D f_\lambda^m - P_Y \xi - D f_\lambda^m \right)
\\&  = (\hat D + \lambda I)^{-1} \left(
          P_Y (\hat\xi - \xi)
          + (\hat D - D) f_\lambda^m
       \right)
\\&  = (\hat D + \lambda I)^{-1} \left(
          P_Y (\hat\xi - \xi)
          + (\hat D - D) (f_\lambda^m - f_0)
          + (\hat D - D) f_0
       \right)
.\end{align}
We thus have,
using $\norm{P_Y} \le 1$,
\[
       \norm*{C^\alpha (f_\lambda^m - f_{\lambda,n}^m)}_\h
   \le \norm*{C^\alpha (P_Y \hat C P_Y + \lambda I)^{-1} P_Y}
       \Big(
          \norm{\hat\xi - \xi}_\h
        + \norm{\hat C - C} \norm{f_\lambda^m - f_0}_\h
        + \norm{\hat C - C} \norm*{f_0}_\h
       \Big)
.\]
We have already bounded $\norm{f_\lambda^m - f_0}_\h$,
and have
$
  \norm{f_0}_\h
  \le \norm{C^\beta} \norm{C^{-\beta} f_0}_\h
  \le R \norm{C}^\beta
$.
Using \cref{thm:gy-props} \ref{part:gy-p} and \ref{part:gy-norm},
we have
\begin{align}
       \norm*{C^\alpha (P_Y \hat C P_Y + \lambda I)^{-1} P_Y}
  &  = \norm*{C^\alpha g_Y(\hat C_\lambda)}
   \le \norm*{C^\alpha \hat C_\lambda^{-\frac12}}
       \norm*{\hat C_\lambda^\frac12 g_Y(\hat C_\lambda) \hat C_\lambda^\frac12}
       \norm*{\hat C_\lambda^{-\frac12}}
\\&\le \frac{1}{\sqrt\lambda} \norm*{C^\alpha \hat C_\lambda^{-\frac12}}
,\end{align}
and so
\[
       \norm*{C^\alpha (f_\lambda^m - f_{\lambda,n}^m)}_\h
   \le \frac{\norm*{C^\alpha \hat C_\lambda^{-\frac12}}}{\sqrt\lambda}
       \Big(
          \norm{\hat\xi - \xi}_\h
        + \norm{\hat C - C} \left( \norm{f_\lambda^m - f_0}_\h + R \norm{C}^\beta \right)
       \Big)
\label{eq:decomp-approx}
.\]
The claim follows by using \cref{eq:decomp-proj,eq:decomp-approx} in \cref{eq:decomp-init}.
\end{proof}

\subsubsection{Remark on unimportance of \texorpdfstring{$\partial_i^2 k(x, \cdot)$}{second-derivative} terms in the basis} \label{remark:y=x}
This decomposition gives some intuition about why terms of the form $\partial_i^2 k(x, \cdot)$,
which are included in the basis of the full-data solution
but missing from our solution even when $Y = X$,
appear to be unimportant (as we also observe empirically).

The only term in the error decomposition depending on the specific basis chosen is the projection error term $\norm{C_\lambda^\frac12 (I - V V^*)}$.
Because the $\partial_i^2 k(x, \cdot)$ directions are not particularly aligned with $C$,
unlike the $\partial_i k(x, \cdot)$ terms,
whether they are included or not should not have a major effect on this term and therefore does not strongly affect the bound.

Moreover, the primary places where \cref{thm:decomp} discards dependence on the basis are that in the estimation error term,
we bounded each of
$\norm{P_Y (\hat\xi - \xi)}$,
$\norm{P_Y (\hat C - C) P_Y}$,
and $\norm{C^\alpha (P_Y \hat C P_Y + \lambda I)^{-1} P_Y}$ terms
by simply dropping the $P_Y$.
For the $C$-based terms,
we again expect that the $\partial_i^2 k(x, \cdot)$ terms do not have a strong effect on the given norms.
Thus the only term that should be very directly affected is
$\norm{P_Y (\hat\xi - \xi)}$;
but since we expect that $\hat\xi \to \xi$ relatively quickly
compared to the convergence of $\hat C \to C$,
this term should not be especially important to the overall error.

\subsection{Probabilistic inequalities} \label{sec:prob}

We only need \cref{thm:decomp} for $\alpha = 0$ and $\alpha = \frac12$;
in the former case,
we use $\norm*{\hat C_\lambda^{-\frac12}} \le 1 / \sqrt{\lambda}$.
Thus we are left with four quantities to control:
$\norm{C^\frac12 \hat C_\lambda^{-\frac12}}$,
$\mathcal C_Y = \norm{C_\lambda^\frac12 (I - V V^*)}^2$,
$\norm{\hat\xi - \xi}_\h$,
and $\norm{\hat C - C}$.

\begin{lemma} \label{thm:c-chat}
  Let $\rho, \delta \in (0, 1)$.
  Under \cref{assump:aapo-1,assump:aapo-2,assump:aapo-3,assump:1-integrable,assump:bounded},
  for any $0 < \lambda \le \frac13 \norm C$,
  we have with probability at least $1 - \delta$ that
  \[
    \norm{C^\frac12 \hat C_\lambda^{-\frac12}}
    \le \frac{1}{\sqrt{1 - \rho}}
  \]
  as long as
  \[
    n \ge \max\left(
      \frac{4}{3 \rho},\
      \frac{40 d \N'_\infty(\lambda)}{\rho^2}
    \right)
    \log \frac{40 \Tr C}{\lambda \delta}
  .\]
\end{lemma}
\begin{proof}
Let $\gamma := \lambda_{\max}\left( C_\lambda^{-\frac12} (C - \hat C) C_\lambda^{-\frac12} \right)$.
\Cref{thm:eigs} gives that
$\norm{C^\frac12 \hat C_\lambda^{-\frac12}} \le \frac{1}{\sqrt{1 - \gamma}}$.
We bound $\gamma$ with \cref{thm:bernstein-sum-ops-all},
using $Y_i^a = \partial_i k(X_a, \cdot)$
so that $\E \sum_{i=1}^d Y_i^a \otimes Y_i^a = C$.
This gives us that $\gamma \le \rho$ with probability at least $1 - \delta$ as long as
\[
    \rho \le \frac{2 w}{3 n} + \sqrt{\frac{10 d \N'_\infty(\lambda) w}{n}}
,\]
which is satisfied by the condition on $n$.
\end{proof}

\begin{lemma} \label{thm:comp-error}
  Sample $m$ points $\{ Y_a \}_{a \in [m]}$ iid from $p_0$,
  and construct a subspace $\h_Y$ from those points in a way determined below;
  let $V V^*$ be the orthogonal projection onto $\h_Y$.
  Choose $\rho, \delta \in (0, 1)$,
  and assume that $\lambda \le \frac13 \norm{C}$.
  Then, under \cref{assump:aapo-1,assump:aapo-2,assump:aapo-3,assump:1-integrable,assump:bounded}
    \[
    \mathcal C_Y
    = \norm{C_\lambda^\frac12 (I - V V^*)}^2
    \le \frac{\lambda}{1 - \rho}
  \]
  with probability at least $1 - \delta$ in each of the following cases:
  \begin{enumerate}[label=(\roman*)]
    \item \label{thm:comp-error:all}
      We put all components of the $m$ points in our basis:
      $Y = \{ \partial_i k(Y_a, \cdot)\}_{a \in [m]}^{i \in [d]}$,
      so that we have $m d$ components.
      We require
      \[
        m \ge \max\left(
          \frac{4}{3 \rho},\
          \frac{40 d \N'_\infty(\lambda)}{\rho^2}
        \right)
        \log\left( \frac{40}{\lambda \delta} \Tr(C) \right)
      .\]
    \item \label{thm:comp-error:bernoulli}
      Include each of the $m d$ components $\partial_i k(Y_a, \cdot)$
      with probability $p$,
      so that the total number of components is distributed randomly as $\mathrm{Binomial}(m d, p)$.
      The statement holds as long as
      \[
        m \ge \max\left(
          \frac{4}{3 \rho},\
          \frac{40 \left(d + \frac1p - 1\right) \N'_\infty(\lambda)}{\rho^2}
        \right)
        \log\left(
          \frac{40}{\lambda \delta} \Tr(C)
          \frac{d + \frac1p - 1}{d + 15\left( \frac1p - 1 \right)}
        \right)
      .\]
    \item \label{thm:comp-error:fixed-num}
      For each of the $m$ data points, we choose $\ell \in [1, d]$ components uniformly at random without replacement, so that we have $m \ell$ components.
      Assume here that $d > 1$; otherwise we necessarily have $\ell = d = 1$, covered by case \ref{thm:comp-error:all}.
      The statement holds as long as
      \[
        m \ge \max\left(
          \frac{4}{3 \rho},\
          \frac{40 d \N'_\infty(\lambda)}{\rho^2}
        \right)
        \log\left( \frac{40}{\lambda \delta} \Tr(C) \left(1 + 14 \frac{d - \ell}{\ell (d - 1)} \right) \right)
      .\]
  \end{enumerate}
\end{lemma}
\begin{proof}
Define the random operator $R_Y : \h \to \R^{md}$ by
$R_Y := \frac{1}{\sqrt m} \sum_{a=1}^m \sum_{i=1}^d \frac{1}{p_{ai}} e_{ai} \otimes \partial_i k(Y_a, \cdot)$,
where $p_{ai}$ is the probability that the corresponding component is included in the basis.
Since $p_{ai} > 0$ for each $(a, i)$ in these setups,
the operator $R_Y$ is bounded.
Note that $\overline{\range{Z^*}} = \range{P_Y} = \h_Y$ and that $ \norm{C_\lambda^\frac12 (I - V V^*)}^2 = \norm{(I - V V^*)C_\lambda^\frac12 }^2 $ as $  C_\lambda^\frac12 $ is symmetric.
Thus we can apply \cref{thm:proj-norm,thm:eigs} to observe that
\[
  \norm{C_\lambda^\frac12 (I - V V^*)}^2
  \le \lambda \norm*{ (R_Y^* R_Y + \lambda I)^{-\frac12} C_\lambda^\frac12}^2
  \le \frac{\lambda}{1 - \lambda_{\max}\left( C_\lambda^{-\frac12} \left( C - R_Y^* R_Y \right) C_\lambda^{-\frac12} \right)}
.\]
It remains to bound the relevant eigenvalue by $\rho$.
We do so with the results of \cref{sec:sumCorrelatedOperators}:
\cref{thm:bernstein-sum-ops-all} for \ref{thm:comp-error:all},
\cref{thm:bernstein-sum-ops-bernoulli} for \ref{thm:comp-error:bernoulli},
and \cref{thm:bernstein-sum-ops-fixed-num} for \ref{thm:comp-error:fixed-num}.
\end{proof}

For the remaining two quantities, we use simple Hoeffding bounds:\footnote{A Bernstein bound would allow for a slightly better result when $\kappa_1$ and $\kappa_2$ are large, at the cost of a more complex form.}

\begin{lemma}[Concentration of $\hat\xi$] \label{thm:xi-hat}
  Under \cref{assump:bounded},
  with probability at least $1 - \delta$ we have
  \[
    \norm{\hat\xi - \xi}_\h
    \le \frac{2 d (Q \kappa_1 + \kappa_2)}{\sqrt n}
        \left( 1 + \sqrt{2 \log\tfrac1\delta} \right)
  .\]
\end{lemma}
\begin{proof}
Let \[
  \nu_a := \sum_{i=1}^d \left( \partial_i \log q_0(X_a) \partial_i k(X_a, \cdot) + \partial_i^2 k(X_a, \cdot) \right) - \xi
,\]
so that $\hat\xi - \xi = \frac{1}{n} \sum_{a=1}^n \nu_a$,
and for each $a$ we have that
$\E \nu_a = 0$
and
\begin{align}
       \norm*{\nu_a}_\h
   \le 2 \sup_{x \in \Omega} \norm*{\sum_{i=1}^d \partial_i \log q_0(x) \partial_i k(x, \cdot) + \partial_i^2 k(x, \cdot)}_\h
   \le 2 d \left( Q \kappa_1 + \kappa_2 \right)
.\end{align}
Applying \cref{thm:hoeffding-vectors} to the vectors $\nu_a$ gives the result.
\end{proof}

\begin{lemma}[Concentration of $\hat C$] \label{thm:c-hat}
  Under \cref{assump:bounded},
  with probability at least $1 - \delta$ we have
  \[
    \norm{\hat C - C}
    \le \frac{2 d \kappa_1^2}{\sqrt n}
        \left( 1 + \sqrt{2 \log\tfrac1\delta} \right)
  .\]
\end{lemma}
\begin{proof}
Let \[
  C_x := \sum_{i=1}^d \partial_i k(x, \cdot) \otimes \partial_i k(x, \cdot)
,\]
so that $\hat C = \frac1n \sum_{a=1}n C_{X_a}$, $C = \E C_x$.
We know that
\begin{align}
       \norm*{C_x - C}
  &\le 2 \sum_{i=1}^d \norm*{\partial_i k(x, \cdot)}_\h^2
   \le 2 d \kappa_1^2
\\
       \norm*{C_x - C}_\hs
  &\le 2 \sum_{i=1}^d \sup_{x \in \Omega} \norm*{\partial_i k(x, \cdot)}_\h^2
   \le 2 d \kappa_1^2
,\end{align}
so applying \cref{thm:hoeffding-operators} shows the result.
\end{proof}

\subsection{Final bound} \label{sec:final-bound}
\begin{theorem}[Finite-sample convergence of $f_{\lambda,n}^m$] \label{thm:finite-bound}
  Under \cref{assump:range-space,assump:aapo-1,assump:aapo-2,assump:aapo-3,assump:well-spec,assump:1-integrable,assump:bounded},
  let $\delta \in (0, 1)$ and define $S_\delta := 1 + \sqrt{2 \log\frac4\delta}$.
  Sample basis $m$ points $\{Y_a\}_{a \in [m]}$ iid from $p_0$,
  not necessarily independent of $X$,
  and choose a basis as:
  \begin{enumerate}[label=(\roman*)]
    \item \label{thm:finite-bound:all}
      All $d$ components $\{ \partial_i k(Y_a, \cdot) \}_{a \in [m]}^{i \in [d]}$:
      set $w := 1$, $r := 0$.
    \item \label{thm:finite-bound:bernoulli}
      A random subset, choosing each of the $m d$ components $\partial_i k(Y_a, \cdot)$ independently with probability $p$:
      set $w := \frac{d p + 1 - p}{d p + 80 \left( 1 - p \right) / 3}$,
      $r := \frac1p - 1$.
    \item \label{thm:finite-bound:fixed-num}
      A random subset,
      choosing $\ell$ components $\partial_i k(Y_a, \cdot)$ uniformly without replacement for each of the $m$ points:
      set $w := 1 + 14 \frac{d - \ell}{\ell (d - 1)}$, $r := 0$.
      (If $d = 1$, use case \ref{thm:finite-bound:all}.)
  \end{enumerate}
  Assume that $0 < \lambda < \frac13 \norm C$.
  When
\[
    m \ge \frac{90 (d + r) \kappa_1^2}{\lambda} \log \frac{160 d \kappa_1^2 w}{\lambda \delta}
    \quad\text{and}\quad
    n \ge \frac{90 d \kappa_1^2}{\lambda} \log \frac{160 d \kappa_1^2}{\lambda \delta}
,\]
  we have with probability at least $1 - \delta$ that both of the following hold simultaneously:
  \begin{align}
    \norm{f_{\lambda,n}^m - f_0}_\h
  &\le 7 R
        \lambda^{\min\left(\frac12,\ \beta \right)}
        (d \kappa_1^2)^{\max\left(0,\ \beta - \frac12\right)}
\\&   + \frac{2 d}{\lambda \sqrt n} \Big(
          Q \kappa_1 + \kappa_2
        + R \kappa_1^2 \left(
            7
            \lambda^{\min\left(\frac12 ,\ \beta \right)}
            (d \kappa_1^2)^{\max\left(0,\ \beta - \frac12\right)}
          + (d \kappa_1^2)^\beta
          \right)
       \Big)
       S_\delta
\\
       \norm{C^\frac12 (f_{\lambda,n}^m - f_0)}_\h
  &\le 7 R
        \lambda^{\min\left(1,\ \beta + \frac12\right)}
        (d \kappa_1^2)^{\max\left(0,\ \beta - \frac12\right)}
\\&   + \frac{2 d \sqrt{3}}{\sqrt{\lambda n}} \Big(
          Q \kappa_1 + \kappa_2
        + R \kappa_1^2 \left(
            7
            \lambda^{\min\left(\frac12,\ \beta \right)}
            (d \kappa_1^2)^{\max\left(0,\ \beta - \frac12\right)}
          + (d \kappa_1^2)^\beta
          \right)
       \Big)
       S_\delta
  .\end{align}
\end{theorem}
\begin{proof}
Recall from \cref{thm:decomp} that
\begin{multline}
    \norm{C^\alpha (f_{\lambda,n}^m - f_0)}_\h
    \le R \left( 2 \mathcal C_Y + \lambda \right) c(\alpha) c(\beta)
\\    + \frac{1}{\sqrt\lambda} \norm*{C^\alpha \hat C_\lambda^{-\frac12}}
        \Big(
          \norm{\hat\xi - \xi}_\h
        + \norm{\hat C - C} R \left(
            \left( \frac{2 \mathcal C_Y}{\sqrt\lambda} + \sqrt\lambda \right) c(\beta)
          + \norm{C}^\beta
          \right)
       \Big)
,\end{multline}
for $c(\alpha) = \lambda^{\min\left(0, \alpha - \frac12 \right)} \norm{C}^{\max\left( 0, \alpha - \frac12 \right)}$.

We'll use a union bound over the results of \cref{thm:xi-hat,thm:c-hat,thm:c-chat,thm:comp-error}.
Note that under \cref{assump:bounded},
each of $\norm C$ and $\Tr C$ are at most $d \kappa_1^2$
and $\N'_\infty(\lambda) \le \kappa_1^2 / \lambda$.

We first use $\rho = \frac23$ in \cref{thm:c-chat,thm:comp-error}
to get that
$\norm{C^\frac12 \hat C_\lambda^{-\frac12}} \le \sqrt{3}$
and $\mathcal C_Y \le 3 \lambda$
with probability at least $\frac\delta2$
when $n$ and $m$ are each at least
\[
    \max\left(
      2,\
      90 (d + r) \N'_\infty(\lambda)
    \right)
    \log \frac{40 \Tr(C) w}{\lambda \frac\delta4}
    \le
    \frac{90 (d + r) \kappa_1^2}{\lambda} \log \frac{160 d \kappa_1^2 w}{\lambda \delta}
,\]
where for $m$ we use $r$ and $w$ as defined in the statement,
and for $n$ we use $r = 0$, $w = 1$;
we also used that $\lambda < \frac13 \norm{C}$ to resolve the $\max$.
The claim follows from applying \cref{thm:xi-hat,thm:c-hat}.
\end{proof}

\cref{thm:main-big-o} now follows from considering the asymptotics of \cref{thm:finite-bound},
once we additionally make \cref{assump:bounded-kernel}:

\begin{proof}[Proof of \cref{thm:main-big-o}]
Let $b := \min\left(\frac12,\ \beta\right)$.
Under \cref{assump:range-space,assump:aapo-1,assump:aapo-2,assump:aapo-3,assump:well-spec,assump:1-integrable,assump:bounded},
as $n \to \infty$ \cref{thm:finite-bound} gives:
\begin{align}
     \norm{f_{\lambda,n}^m - f_0}_\h
  &= \bigO_{p_0}\left( \lambda^b + n^{-\frac12} \lambda^{-1} + n^{-\frac12} \lambda^{b - 1} \right)
   = \bigO_{p_0}\left( \lambda^b + n^{-\frac12} \lambda^{-1} \right)
\\   \norm{C^\frac12 (f_{\lambda,n}^m - f_0)}_\h
  &= \bigO_{p_0}\left( \lambda^{b + \frac12} + n^{-\frac12} \lambda^{-\frac12} + n^{-\frac12} \lambda^{b - \frac12} \right)
  = \bigO_{p_0}\left( \lambda^{b + \frac12} + n^{-\frac12} \lambda^{-\frac12} \right)
\end{align}
as long as $\min(n, m) = \bigOmega(\lambda^{-1} \log \lambda^{-1})$.
Choosing $\lambda = n^{-\theta}$,
this requirement is $\min(n, m) = \bigOmega(n^\theta \log n)$ and the bounds become
\begin{align}
     \norm{f_{\lambda,n}^m - f_0}_\h
  &= \bigO_{p_0}\left( n^{-b \theta} + n^{\theta - \frac12} \right)
\\   \norm{C^\frac12 (f_{\lambda,n}^m - f_0)}_\h
  &= \bigO_{p_0}\left( n^{-b \theta - \frac12 \theta} + n^{\frac12\theta - \frac12} \right)
.\end{align}
Both bounds are minimized when $\theta = \frac{1}{2(1 + b)}$,
which since $0 \le b \le \frac12$ leads to $\frac12 \ge \theta \ge \frac13$,
and the requirement on $n$ is always satisfied once $n$ is large enough.
This shows, as claimed, that
\[
    \norm{f_{\lambda,n}^m - f_0}_\h = \bigO_{p_0}\left( n^{-\frac{b}{2(b+1)}} \right)
    \qquad
    J(p_0 \| p_{f_{\lambda,n}^m}) = \bigO_{p_0}\left( n^{-\frac{2b+1}{2(b+1)}} \right)
\]
when $m = \bigOmega\left( n^{\frac{1}{2(1+b)}} \log n \right)$.

The bounds on $L^r$, Hellinger, and KL convergence follow from \cref{thm:dist-distances}
under \cref{assump:bounded-kernel}.
\end{proof}

\section{Auxiliary results} \label{sec:aux}

\subsection{Standard concentration inequalities in Hilbert spaces}
\begin{lemma}[Hoeffding-type inequality for random vectors] \label{thm:hoeffding-vectors}
Let $X_1, \dots, X_n$ be iid random variables in a (separable) Hilbert space,
where $\E X_i = 0$ and $\norm{X_i} \le L$ almost surely. Then for any $\varepsilon > L / \sqrt{n}$,
\[
  \Pr\left( \norm*{\frac1n \sum_{i=1}^n X_i} > \varepsilon \right)
  \le \exp\left( - \frac12 \left(\frac{\sqrt{n} \varepsilon}{L} - 1 \right)^2 \right)
;\]
equivalently, we have with probability at least $1 - \delta$ that
\[
  \norm*{\frac1n \sum_{i=1}^n X_i}
  \le \frac{L}{\sqrt n} \left( 1 + \sqrt{2 \log\tfrac1\delta} \right)
.\]
\end{lemma}
\begin{proof}
Following Example 6.3 of \citet{boucheron},
we can apply McDiarmid's inequality.
The function $f(X_1, \dots, X_n) = \norm*{\frac1n \sum_{i=1}^n X_i}$
satisfies bounded differences:
\[
  \abs*{\norm*{\frac1n \sum_{i=1}^n X_i} - \norm*{\frac1n \hat X_1 + \frac1n \sum_{i=2}^n X_i}}
  \le \norm*{\frac1n (X_1 - \hat X_1)}
  \le \frac{2 L}{n}
.\]
Thus for $\varepsilon \ge \E \norm*{\frac1n \sum_i X_i}$,
\[
  \Pr\left(
    \norm*{\frac1n \sum_i X_i} > \varepsilon
  \right)
  \le \exp\left( - \frac{n \left(\varepsilon - \E\norm*{\frac1n \sum_i X_i} \right)^2}{2 L^2} \right)
.\]
We also know that
\[
  \E \norm*{\frac1n \sum_i X_i}
  \le \frac1n \sqrt{\E \norm*{\sum_i X_i}^2}
  = \frac1n \sqrt{\sum_{i,j} \E \langle X_i, X_j \rangle}
  = \frac1n \sqrt{\sum_i \E \norm{X_i}^2}
  \le \frac1n \sqrt{n L^2}
  = \frac{L}{\sqrt n}
,\]
so
\[
  \Pr\left(
    \norm*{\frac1n \sum_i X_i} > \varepsilon
  \right)
  \le \exp\left( - \frac{n \left(\varepsilon - \frac{L}{\sqrt n} \right)^2}{2 L^2} \right)
  =   \exp\left( - \frac12 \left( \frac{\sqrt n \varepsilon}{L} - 1 \right)^2 \right)
\]
as desired.
The second statement follows by simple algebra.
\end{proof}

\begin{lemma}[Hoeffding-type inequality for random Hilbert-Schmidt operators] \label{thm:hoeffding-operators}
Let $X_1, \dots, X_n$ be iid random operators in a (separable) Hilbert space,
where $\E X_i = 0$ and $\norm{X_i} \le L$, $\norm{X_i}_\hs \le B$ almost surely.
Then for any $\varepsilon > B / \sqrt n$,
\[
  \Pr\left( \norm*{\frac1n \sum_{i=1}^n X_i} < \varepsilon \right)
  \le \exp\left( - \frac12 \left(\frac{\sqrt{n} \varepsilon}{L} - \frac{B}{L} \right)^2 \right)
;\]
equivalently, we have with probability at least $1 - \delta$ that
\[
  \norm*{\frac1n \sum_{i=1}^n X_i}
  \le \frac{1}{\sqrt n} \left( B + L \sqrt{2 \log\tfrac1\delta} \right)
.\]
\end{lemma}
\begin{proof}
The argument is the same as \cref{thm:hoeffding-vectors}, except that
\[
  \E \norm*{\frac1n \sum_i X_i}
  \le \frac1n \sqrt{\E \norm*{\sum_i X_i}_\hs^2}
  = \frac1n \sqrt{\sum_{i,j} \E \langle X_i, X_j \rangle_\hs}
  = \frac1n \sqrt{\sum_i \E \norm{X_i}_\hs^2}
  \le
      \frac{B}{\sqrt n}
\]
using $\norm{X_i} \le \norm{X_i}_\hs$.
\end{proof}

\begin{lemma}[Bernstein's inequality for a sum of random operators; Proposition 12 of \citet{lessismore}] \label{thm:bernstein-ops}
  Let $\h$ be a separable Hilbert space,
  and $X_1, \dots, X_n$ a sequence of iid self-adjoint positive random operators on $\h$,
  with $\E X_1 = 0$, $\lambda_{\max}(X_1) \le L$ almost surely for some $L > 0$.
  Let $S$ be a positive operator such that $\E[ X_1^2 ] \preceq S$.
  Let $\beta = \log \frac{2 \Tr S}{\norm{S} \delta}$.
  Then for any $\delta \ge 0$, with probability at least $1 - \delta$
  \[
    \lambda_{\max}\left( \frac1n \sum_{i=1}^n X_i \right)
    \le \frac{2 L \beta}{3 n} + \sqrt{\frac{2 \norm S \beta}{n}}
  .\]
\end{lemma}

\subsection{Concentration of sum of correlated operators} \label{sec:sumCorrelatedOperators}
The following result is similar to Proposition 8 of \citet{lessismore},
but the proof is considerably more complex due to the sum over correlated operators.

We also allow for a random ``masking'' operation via the $U_i^a$.
\Cref{thm:bernstein-sum-ops-gen} applies to general sampling schemes $U_i^a$;
\cref{thm:bernstein-sum-ops-bernoulli,thm:bernstein-sum-ops-fixed-num,thm:bernstein-sum-ops-all}
specialize it to particular sampling schemes.

\begin{lemma} \label{thm:bernstein-sum-ops-gen}
  Let $W_a = (Y_i^a)_{i \in [d]}$ be a random $d$-tuple of vectors in a separable Hilbert space $\h$,
  with $\{ W_a \}_{a \in [n]}$ iid.

  Let $U^a = (U_i^a)_{i \in d}$ be a corresponding $d$-tuple of random vectors,
  with $\Pr(U_i^a \in \{0, 1\}) = 1$,
  such that the $\{U^a\}_{a \in [n]}$ are iid,
  $\E[U_i^a] := \mu_i \in (0, 1]$,
  and $U^a$ is independent of $W^a$.
  Define $\nu_{ij} := \E[U_i^a U_j^a] / (\mu_i \mu_j)$,
  $\nu_i = \sum_{j=1}^d \nu_{ij}$.

  Suppose that $Q = \E \sum_{i=1}^d Y_i^1 \otimes Y_i^1$ exists and is trace class,
  and that for any $\lambda > 0$ there is $\N'_\infty(\lambda) < \infty$
  such that $\langle Y_l^a, (Q + \lambda I)^{-1} Y_l^a \rangle_\h \le \N'_\infty(\lambda)$ almost surely.
  Let $Q_\lambda = Q + \lambda I$,
  $V_a = \sum_{i=1}^d \frac{1}{\mu_i} U_i^a (Y_i^a \otimes Y_i^a)$.

  Let
  \[
    S := \N'_\infty(\lambda)
         Q_\lambda^{-\frac12}
         \left( 2 \E\left[ \sum_{i,j}^d \nu_{ij} (Y_i \otimes Y_j) \right] + 3 \E\left[ \sum_{i=1}^d \nu_i (Y_i \otimes Y_i) \right] \right)
         Q_\lambda^{-\frac12}
  ,\]
  and suppose that $\Tr S \le t$, $s_* \le \norm S \le s^*$.
  (These bounds will depend on the distribution of $U^a$.)

  Then with probability at least $1 - \delta$ we have that
  \[
    \lambda_{\max}\left( Q_\lambda^{-\frac12} \left(Q - \frac1n \sum_{a=1}^n V_a \right) Q_\lambda^{-\frac12} \right)
    \le \frac{2 \beta}{3 n} + \sqrt{\frac{2 s^* \beta}{n}}
    ,\qquad
    \beta = \log\left( \frac{2 t}{\delta s_*} \right)
  .\]

\end{lemma}
\begin{proof}
We will apply the Bernstein inequality for random operators, \cref{thm:bernstein-ops},
to $Z_a := Q_\lambda^{-\frac12} (Q - V_a) Q_\lambda^{-\frac12}$.
For each $a$,
\[
  \E V_a = \sum_{i=1}^d \frac{\E U_i^a}{\mu_i} \E[Y_i^a \otimes Y_i^a] = Q
\]
so that $\E Z_a = 0$,
and since $V_a$ is positive and $Q_\lambda$ is self-adjoint,
\[
    \sup_{\norm{f}_\h = 1} \langle f, Z_a f \rangle_\h
  = \sup_{\norm{f}_\h = 1} \langle f, Q_\lambda^{-1} Q f \rangle_\h - \langle f, Q_\lambda^{-\frac12} V_a Q_\lambda^{-\frac12} f \rangle_\h
  \le \sup_{\norm{f}_\h = 1} \langle f, Q_\lambda^{-1} Q f \rangle_\h
  \le 1
.\]
To apply \cref{thm:bernstein-ops},
we now need to show that the positive operator $S$ upper bounds the second moment of $Z_a$.
Letting $u \in \h$,
and dropping the subscript $a$ for brevity, we have that
\begin{align}
       \langle u, \E[ Z^2 ] u \rangle_\h
  &  = \left\langle
          u,
          \E[ Q_\lambda^{-\frac12} V Q_\lambda^{-1} V Q_\lambda^{-\frac12}] u
       \right\rangle_\h
     - \left\langle
          u,
          Q_\lambda^{-\frac12} Q Q_\lambda^{-1} Q Q_\lambda^{-\frac12} u
       \right\rangle_\h
\\&\le \left\langle
          u,
          Q_\lambda^{-\frac12} \E[ V Q_\lambda^{-1} V ] Q_\lambda^{-\frac12} u
       \right\rangle_\h
\\&  = \left\langle
          Q_\lambda^{-\frac12} u,
          \E[ V Q_\lambda^{-1} V ] Q_\lambda^{-\frac12} u
       \right\rangle_\h
\\&  = \sum_{i, j}^d
       \left\langle
          Q_\lambda^{-\frac12} u,
          \E\left[ \frac{U_i}{\mu_i} (Y_i \otimes Y_i) Q_\lambda^{-1} (Y_j \otimes Y_j) \frac{U_j}{\mu_j} \right] Q_\lambda^{-\frac12} u
       \right\rangle_\h
\\&  = \sum_{i, j}^d
       \frac{\E[U_i U_j]}{\mu_i \mu_j}
       \E\left[
          \langle Q_\lambda^{-\frac12} u, Y_i \rangle_\h
          \langle Q_\lambda^{-\frac12} u, Y_j \rangle_\h
          \langle Y_i, Q_\lambda^{-1} Y_j \rangle_\h
       \right]
.\end{align}
Let $\nu_{ij} = \E[U_i U_j] / (\mu_i \mu_j)$.
Using
$
  2 \langle x, A y \rangle
  = \langle x + y, A (x + y) \rangle - \langle x, A x \rangle - \langle y, A y \rangle
$,
we get:
\begin{align}
       \langle u, \E[ Z^2 ] u \rangle_\h
  &\le \frac12 \sum_{i, j}^d
       \nu_{ij}
       \E\left[
          \langle Q_\lambda^{-\frac12} u, Y_i \rangle_\h
          \langle Q_\lambda^{-\frac12} u, Y_j \rangle_\h
          \langle Y_i + Y_j, Q_\lambda^{-1} (Y_i + Y_j) \rangle_\h
       \right]
\\&\qquad
     - \sum_{i, j}^d
       \nu_{ij}
       \E\left[
          \langle Q_\lambda^{-\frac12} u, Y_i \rangle_\h
          \langle Q_\lambda^{-\frac12} u, Y_j \rangle_\h
          \langle Y_i, Q_\lambda^{-1} Y_i \rangle
       \right]
.\end{align}
Similarly using
$2 \langle A, x \rangle \langle A, y \rangle
 = \langle A, x + y \rangle^2 - \langle A, x \rangle^2 - \langle A, y \rangle^2$,
we get that the first line is
\begin{multline}
       \frac14 \sum_{i, j}^d \nu_{ij} \left(
         \E\left[
            \langle Q_\lambda^{-\frac12} u, Y_i + Y_j \rangle_\h^2
            \langle Y_i + Y_j, Q_\lambda^{-1} (Y_i + Y_j) \rangle_\h
         \right]
\right.\\\left.
      - \E\left[
            \langle Q_\lambda^{-\frac12} u, Y_i \rangle_\h^2
            \langle Y_i + Y_j, Q_\lambda^{-1} (Y_i + Y_j) \rangle_\h
         \right]
      - \E\left[
            \langle Q_\lambda^{-\frac12} u, Y_j \rangle_\h^2
            \langle Y_i + Y_j, Q_\lambda^{-1} (Y_i + Y_j) \rangle_\h
         \right]
      \right)
\end{multline}
and the second is
\begin{multline}
\frac12 \sum_{i, j}^d \nu_{ij} \left(
     -
       \E\left[
          \langle Q_\lambda^{-\frac12} u, Y_i + Y_j \rangle_\h^2
          \langle Y_i, Q_\lambda^{-1} Y_i \rangle
       \right]
\right.\\\left.
     +
       \E\left[
          \langle Q_\lambda^{-\frac12} u, Y_i \rangle_\h^2
          \langle Y_i, Q_\lambda^{-1} Y_i \rangle
       \right]
     +
       \E\left[
          \langle Q_\lambda^{-\frac12} u, Y_j \rangle_\h^2
          \langle Y_i, Q_\lambda^{-1} Y_i \rangle
       \right]
    \right)
.\end{multline}
Each of these expectations is nonnegative,
so dropping the ones with negative coefficients gives:
\begin{multline}
       \langle u, \E[ Z^2 ] u \rangle_\h
   \le \frac14 \sum_{i, j}^d \nu_{ij}
       \E\left[
          \langle Q_\lambda^{-\frac12} u, Y_i + Y_j \rangle_\h^2
          \langle Y_i + Y_j, Q_\lambda^{-1} (Y_i + Y_j) \rangle_\h
       \right]
\\
     + \frac12 \sum_{i, j}^d \nu_{ij}
       \E\left[
          \langle Q_\lambda^{-\frac12} u, Y_i \rangle_\h^2
          \langle Y_i, Q_\lambda^{-1} Y_i \rangle
       \right]
     + \frac12 \sum_{i, j}^d \nu_{ij}
       \E\left[
          \langle Q_\lambda^{-\frac12} u, Y_j \rangle_\h^2
          \langle Y_i, Q_\lambda^{-1} Y_i \rangle
       \right]
.\end{multline}
Recalling that $\langle Y_i, Q_\lambda^{-1} Y_i \rangle \le \N'_\infty(\lambda)$,
the second line is upper-bounded by $\N'_\infty(\lambda)$ times
\[
       \frac12 \sum_{i, j}^d \nu_{ij}
       \E\left[
          \langle Q_\lambda^{-\frac12} u, Y_i \rangle_\h^2
       \right]
     + \frac12 \sum_{i, j}^d \nu_{ij}
       \E\left[
          \langle Q_\lambda^{-\frac12} u, Y_j \rangle_\h^2
       \right]
   = \sum_{i=1}^d \nu_i
       \E\left[
          \langle Q_\lambda^{-\frac12} u, Y_i \rangle_\h^2
       \right]
,\]
where $\nu_i = \sum_{j=1}^d \nu_{ij}$.
We also have that
\[
  \langle Y_i + Y_j, Q_\lambda^{-1} (Y_i + Y_j) \rangle_\h
  = \norm{Q_\lambda^{-\frac12} (Y_i + Y_j)}_\h^2
  \le 2( \norm{Q_\lambda^{-\frac12} Y_i}_\h^2 + \norm{Q_\lambda^{-\frac12} Y_j}_\h^2)
  \le 4 \N'_\infty(\lambda)
,\]
so the first sum is at most $\N'_\infty(\lambda)$ times
\begin{align}
       \sum_{i, j}^d \nu_{ij}
  &    \E\left[
          \langle Q_\lambda^{-\frac12} u, Y_i + Y_j \rangle_\h^2
       \right]
\\&  = \sum_{i, j}^d \nu_{ij}
       \E\left[
          \langle Q_\lambda^{-\frac12} u, Y_i \rangle_\h^2
        + \langle Q_\lambda^{-\frac12} u, Y_j \rangle_\h^2
        + 2 \langle Q_\lambda^{-\frac12} u, Y_i \rangle_\h \langle Q_\lambda^{-\frac12} u, Y_j \rangle_\h
       \right]
\\&  = 2 \sum_{i=1}^d \nu_{i}
       \E\left[
          \langle Q_\lambda^{-\frac12} u, Y_i \rangle_\h^2
       \right]
     + 2 \sum_{i, j}^d \nu_{ij}
       \E\left[
            \langle Q_\lambda^{-\frac12} u, Y_i \rangle_\h \langle Q_\lambda^{-\frac12} u, Y_j \rangle_\h
       \right]
.\end{align}
Thus
\begin{align}
       \langle u, \E[ Z^2 ] u \rangle_\h
  &\le \N'_\infty(\lambda) \left(
       2 \sum_{i, j}^d \nu_{ij}
       \E\left[
          \langle Q_\lambda^{-\frac12} u, Y_i \rangle_\h \langle Q_\lambda^{-\frac12} u, Y_j \rangle_\h
       \right]
     + 3 \sum_{i=1}^d \nu_i
       \E\left[
          \langle Q_\lambda^{-\frac12} u, Y_i \rangle_\h^2
       \right]
     \right)
\\&  = \left\langle u,
         \N'_\infty(\lambda)
         Q_\lambda^{-\frac12}
         \left( 2 \E\left[ \sum_{i,j}^d \nu_{ij} (Y_i \otimes Y_j) \right] + 3 \E\left[ \sum_{i=1}^d \nu_i (Y_i \otimes Y_i) \right] \right)
         Q_\lambda^{-\frac12} u \right\rangle_\h
\\&  = \left\langle u, S u \right\rangle_\h
,\end{align}
recalling that
\[
  S  = \N'_\infty(\lambda)
       Q_\lambda^{-\frac12}
       \left( 2 \E\left[ \sum_{i,j}^d \nu_{ij} (Y_i \otimes Y_j) \right] + 3 \E\left[ \sum_{i=1}^d \nu_i (Y_i \otimes Y_i) \right] \right)
       Q_\lambda^{-\frac12}
.\]
Thus we have the desired upper bound $\E[Z^2] \preceq S$.

Recall that $\Tr S \le t$, $s_* \le \norm S \le s^*$.
Then by \cref{thm:bernstein-ops},
with probability at least $1 - \delta$ we have that
\[
  \lambda_{\max}\left( \frac1n Z_a \right)
  \le \frac{2 \beta'}{3 n} + \sqrt{\frac{2 \norm S \beta'}{n}}
  \le \frac{2 \beta}{3 n} + \sqrt{\frac{2 s^* \beta}{n}}
,\]
where
\[
  \beta'
   := \log \frac{2 \Tr S}{\delta \norm S}
  \le \log\frac{2 t}{\delta s_*}
   =: \beta
,\]
as desired.
\end{proof}

We will now find $t$, $s_*$, $s^*$ for some particular sampling schemes.
The following initial lemma will be useful for this purpose:

\begin{lemma} \label{thm:berstein-sum-ops-helper}
In the setup of \cref{thm:bernstein-sum-ops-gen},
define
$M := \E\left[ \left( \sum_{i=1}^d Y_i \right) \otimes \left( \sum_{i=1}^d Y_i \right) \right]$.
We have:
\[
  M \preceq d Q
,\quad
  \Tr\left( Q_\lambda^{-\frac12} M Q_\lambda^{-\frac12} \right) \le \frac{d}{\lambda} \Tr(Q)
,\quad
  \norm*{Q_\lambda^{-\frac12} M Q_\lambda^{-\frac12}} \le d
.\]
\end{lemma}
\begin{proof}
We first show $M \preceq d Q$:
\begin{align}
       \left\langle u, M u\right\rangle_\h
  &  = \left\langle u, \E\left[ \left(\sum_{i=1}^d Y_i \right) \otimes \left(\sum_{i=1}^d Y_i\right) \right] u \right\rangle_\h
     = \E\left[ \left\langle u, \sum_{i=1}^d Y_i \right\rangle_\h^2 \right]
\\&\le \E\left[ d \sum_{i=1}^d \left\langle u, Y_i \right\rangle_\h^2 \right]
     = \E\left[ d \sum_{i=1}^d \left\langle u, (Y_i \otimes Y_i) u \right\rangle_\h \right]
     = \langle u, d Q u \rangle_\h
.\end{align}

Thus $\Tr(M) \le d \Tr(Q)$, and since $\norm{Q_\lambda^{-1}} \le \frac1\lambda$ we have
\[
  \Tr\left( Q_\lambda^{-\frac12} M Q_\lambda^{-\frac12} \right)
  = \Tr\left( Q_\lambda^{-1} M \right)
  \le \frac{1}{\lambda} \Tr(M)
  \le \frac{d}{\lambda} \Tr(Q)
.\]

For any $u$ with $\norm{u}_\h = 1$:
\[
       \langle u, Q_\lambda^{-\frac12} M Q_\lambda^{-\frac12} u \rangle_\h
     = \langle Q_\lambda^{-\frac12} u, M (Q_\lambda^{-\frac12} u) \rangle_\h
   \le \langle Q_\lambda^{-\frac12} u, d Q (Q_\lambda^{-\frac12} u) \rangle_\h
     = d \langle u, Q Q_\lambda^{-1} u \rangle_\h
   \le d
,\]
and so the norm inequality follows.
\end{proof}

\begin{lemma} \label{thm:bernstein-sum-ops-bernoulli}
  Take the setup of \cref{thm:bernstein-sum-ops-gen}
  where each $U_i^a$ is independently distributed as $\mathrm{Bernoulli}(p)$, for $p \in (0, 1]$.
  The number of sampled components is random, distributed as $\mathrm{Binomial}(nd, p)$.

  For any $\rho \in (0, \frac12)$,
  $\lambda \in (0, \rho \norm{Q}]$,
  and $\delta \ge 0$,
  it holds with probability at least $1 - \delta$ that
  \[
    \lambda_{\max}\left( Q_\lambda^{-\frac12} \left( Q - \frac1n \sum_{a=1}^n V_a \right) Q_\lambda^{-\frac12} \right)
    \le \frac{2 \beta}{3 n} + \sqrt{\frac{10 \left(d + 1/p - 1 \right) \N'_\infty(\lambda) \beta}{n}}
  \]
  where
  \[
    \beta := \log\frac{10 \left(d + 1/p - 1 \right) \Tr Q}{\lambda \delta \left(\frac{5/p - 5 + 3d}{1 + \rho} - 2 d \right)}
  .\]
\end{lemma}
\begin{proof}
  Here we have for $i \ne j$
  \[
    \mu_i = p
    ,\qquad
    \nu_{ii} = \frac{\E[U_i^2]}{\mu_i^2} = \frac{1}{\mu_i} = \frac{1}{p}
    ,\qquad
    \nu_{ij} = \frac{\E[U_i U_j]}{\mu_i \mu_j} = \frac{\E U_i}{\mu_i} \frac{\E U_j}{\mu_j} = 1
  .\]
  Define $r := \frac1p - 1$;
  then $\nu_i = r + d$.
  Using \cref{thm:berstein-sum-ops-helper},
  we get that
  \[
    \E\left[ \sum_{i=1}^d \nu_i (Y_i \otimes Y_i) \right]
    = (r + d) Q
  \]
  and
  \[
    \E\left[ \sum_{i,j}^d \nu_{ij} (Y_i \otimes Y_j) \right]
    = \E\left[ \sum_{i,j}^d Y_i \otimes Y_j \right] + \left(\tfrac1p - 1 \right) \E\left[ \sum_{i=1}^d Y_i \otimes Y_i \right]
    = M + r Q
  ,\]
  so that
  \begin{align}
    S &= \N'_\infty(\lambda) Q_\lambda^{-\frac12} \left( 2 (M + r Q) + 3 (r + d) Q \right) Q_\lambda^{-\frac12}
    \\&= \N'_\infty(\lambda) Q_\lambda^{-\frac12} \left( 2 M + (5 r + 3 d) Q \right) Q_\lambda^{-\frac12}
  .\end{align}
  Thus
  \begin{align}
    \Tr S
    &  = \N'_\infty(\lambda) \left( 2 \Tr( Q_\lambda^{-1} M ) + (5 r + 3 d) \Tr(Q_\lambda^{-1} Q) \right)
  \\&\le \frac{5(r + d)}{\lambda} \N'_\infty(\lambda) \Tr(Q)
  .\end{align}
  Likewise, since $\norm*{Q_\lambda^{-\frac12} M Q_\lambda^{-\frac12}} \le d$,
  \[
    \norm S
    \le \N'_\infty(\lambda) \left(
      2 \norm*{Q_\lambda^{-\frac12} M Q_\lambda^{-\frac12}}
      + (3d + 5r) \norm{Q Q_\lambda^{-1}}
    \right)
    \le 5 (d + r) \N'_\infty(\lambda)
  .\]
  Since we have $\lambda \le \rho \norm Q$,
  $\norm{Q Q_\lambda^{-1}} = \frac{\norm Q}{\norm Q + \lambda} \ge \frac{1}{1 + \rho}$
  and so
  \begin{align}
         \norm S
    &  = \N'_\infty(\lambda) \norm*{(5r + 3d) Q Q_\lambda^{-1} - 2 Q_\lambda^{-\frac12} M Q_\lambda^{-\frac12}}
  \\&\ge \N'_\infty(\lambda) \left( (5r + 3d) \norm{Q Q_\lambda^{-1}} - 2 \norm*{Q_\lambda^{-\frac12} M Q_\lambda^{-\frac12}} \right)
  \\&\ge \N'_\infty(\lambda) \left( \frac{5r + 3d}{1 + \rho} - 2 d \right)
  .\end{align}
  This bound is positive when
  $\frac{5r + 3 d}{1 + \rho} > 2 d$,
  i.e.\ $\rho < \frac12 \left(\frac{5 r}{d} + 1\right)$;
  it suffices that $\rho < \frac12$.

  Applying \cref{thm:bernstein-sum-ops-gen} proves the result.
\end{proof}

\begin{lemma} \label{thm:bernstein-sum-ops-fixed-num}
  Take the setup of \cref{thm:bernstein-sum-ops-gen}
  where each $U^a$
  is chosen uniformly from the set of binary vectors with $\norm{U^a}_1 = \ell \in [1, d]$,
  i.e.\ we choose $\ell$ components of each vector at random without replacement.
  Assume that $d > 1$; otherwise, we simply have $\ell = d = 1$,
  which is covered by \cref{thm:bernstein-sum-ops-bernoulli} with $p = 1$.

  For any $\rho \in (0, \frac12)$,
  $\lambda \in (0, \rho \norm{Q})$,
  and $\delta \ge 0$,
  it holds with probability at least $1 - \delta$ that
  \[
    \lambda_{\max}\left( Q_\lambda^{-\frac12} \left( Q - \frac1n \sum_{a=1}^n V_a \right) Q_\lambda^{-\frac12} \right)
    \le \frac{2 \beta}{3 n} + \sqrt{\frac{10 d \N'_\infty(\lambda) \beta}{n}}
  \]
  where
  \[
    \beta := \log\frac{10 \Tr(Q)}{\lambda \delta \left(\left( 3 + 2 \frac{d - \ell}{\ell (d - 1)} \right) \frac{1}{1 + \rho} - 2 \frac{d (\ell - 1)}{\ell (d - 1)}\right)}
  .\]
\end{lemma}
\begin{proof}
  In this case, for $i \ne j$ we have
  \[
    \mu_i = \frac{\ell}{d}
    ,\qquad
    \nu_{ii} = \frac{\E[U_i^2]}{\mu_i^2} = \frac{1}{\mu_i} = \frac{d}{\ell}
    ,\qquad
    \nu_{ij} = \frac{\Pr(U_i = U_j = 1)}{\mu_i \mu_j}
             = \frac{\binom{d-2}{\ell-2}}{\binom{d}{\ell}} \frac{d^2}{\ell^2}
             = \frac{d (\ell - 1)}{\ell (d - 1)}
  .\]
  Thus
  \[
    \nu_i
    = \frac{d}{\ell} + (d - 1) \frac{d (\ell - 1)}{\ell (d - 1)}
    = \frac{d}{\ell} \left( 1 + (\ell - 1) \right)
    = d
  ,\]
  and $\E\left[ \sum_{i=1}^d \nu_i (Y_i \otimes Y_i) \right] = d Q$,
  while
  \begin{align}
       \E\left[ \sum_{i,j}^d \nu_{ij} (Y_i \otimes Y_j) \right]
    &= \frac{d(\ell - 1)}{\ell (d-1)} \E\left[ \sum_{i,j}^d Y_i \otimes Y_j \right]
     + \left(\frac{d}{\ell} - \frac{d (\ell - 1)}{\ell (d-1)} \right) \E\left[ \sum_{i=1}^d Y_i \otimes Y_i \right]
  \\&= \frac{d (\ell - 1)}{\ell (d - 1)} M + \frac{d (d - \ell)}{\ell (d - 1)} Q
  \end{align}
  using $M$ from \cref{thm:berstein-sum-ops-helper},
  and so
  \[
    S = \N'_\infty(\lambda) Q_\lambda^{-\frac12} \left( 2 \frac{d (\ell - 1)}{\ell (d - 1)} M + d \left( 3 + 2 \frac{d - \ell}{\ell (d - 1)} \right) Q \right) Q_\lambda^{-\frac12}
  .\]
  Thus
  \begin{align}
         \Tr S
    &  = \N'_\infty(\lambda) \left( 2 \frac{d (\ell - 1)}{\ell (d - 1)} \Tr(Q_\lambda^{-\frac12} M Q_\lambda^{-\frac12}) + d \left( 3 + 2 \frac{d - \ell}{\ell (d - 1)} \right) \Tr(Q Q_\lambda^{-1}) \right)
  \\&\le \frac1\lambda \N'_\infty(\lambda) \left( 2 \frac{d (\ell - 1)}{\ell (d - 1)} d + d \left( 3 + 2 \frac{d - \ell}{\ell (d - 1)} \right) \right) \Tr(Q)
  \\&  = \frac1\lambda \N'_\infty(\lambda) d \left( 2 \frac{d (\ell - 1) + d - \ell}{\ell (d - 1)} + 3 \right) \Tr(Q)
  \\&  = \frac{5 d}{\lambda} \N'_\infty(\lambda) \Tr(Q)
  .\end{align}
  We similarly have
  \begin{align}
         \norm S
    &\le \N'_\infty(\lambda) \left( 2 \frac{d (\ell - 1)}{\ell (d - 1)} \norm*{Q_\lambda^{-\frac12} M Q_\lambda^{-\frac12}} + d \left( 3 + 2 \frac{d - \ell}{\ell (d - 1)} \right) \norm*{Q Q_\lambda^{-1}} \right)
  \\&\le \N'_\infty(\lambda) \left( 2 \frac{d (\ell - 1)}{\ell (d - 1)} d + d \left( 3 + 2 \frac{d - \ell}{\ell (d - 1)} \right) \right)
  \\&  = 5 d \N'_\infty(\lambda)
  .\end{align}
  Note also that $1 \le \ell \le d$
  implies $2 \frac{d (\ell - 1)}{\ell (d - 1)} \le 3 + 2 \frac{d - \ell}{\ell (d-1)}$
  for integral $\ell$ and $d$.
  Since $M \le d Q$,
  and like in \cref{thm:bernstein-sum-ops-bernoulli} we have that $\norm{Q Q_\lambda^{-1}} \ge \frac{1}{1 + \rho}$,
  we obtain that
  \begin{align}
         \norm S
    &  = \N'_\infty(\lambda) \norm*{- 2 \frac{d (\ell - 1)}{\ell (d - 1)} Q_\lambda^{-\frac12} M Q_\lambda^{-\frac12} + d \left( 3 + 2 \frac{d - \ell}{\ell (d - 1)} \right) Q Q_\lambda^{-1}}
  \\&\ge \N'_\infty(\lambda) \left( - 2 \frac{d (\ell - 1)}{\ell (d - 1)} \norm*{Q_\lambda^{-\frac12} M Q_\lambda^{-\frac12}} + d \left( 3 + 2 \frac{d - \ell}{\ell (d - 1)} \right) \norm*{Q Q_\lambda^{-1}} \right)
  \\&\ge \N'_\infty(\lambda) \left( - 2 \frac{d (\ell - 1)}{\ell (d - 1)} d + d \left( 3 + 2 \frac{d - \ell}{\ell (d - 1)} \right) \frac{1}{1 + \rho} \right)
  \\&  = d \N'_\infty(\lambda) \left( \left( 3 + 2 \frac{d - \ell}{\ell (d - 1)} \right) \frac{1}{1 + \rho} - 2 \frac{d (\ell - 1)}{\ell (d - 1)} \right)
  .\end{align}
  We then have that
  \[
    \frac{t}{s_*}
    = \frac{5 \Tr(Q) / \lambda}{\left( 3 + 2 \frac{d - \ell}{\ell (d - 1)} \right) \frac{1}{1 + \rho} - 2 \frac{d (\ell - 1)}{\ell (d - 1)}}
  ,\]
  which is well-defined and positive as long as
  either $\ell = 1$
  or $\left( 3 + 2 \frac{d - \ell}{\ell (d - 1)} \right) \frac{1}{1 + \rho} > 2 \frac{d (\ell - 1)}{\ell (d - 1)}$,
  i.e.\ $\rho < \frac12 \frac{\ell + 4 - 5 \frac{\ell}{d}}{\ell - 1}$;
  since $\frac{\ell}{d} \le 1$,
  it suffices that $\rho < \frac12$.
  The claim follows from \cref{thm:bernstein-sum-ops-gen}.
\end{proof}
An interesting special case of \cref{thm:bernstein-sum-ops-fixed-num} is $\ell = 1$, where $t / s_*$ reduces to $\frac{1 + \rho}{\lambda} \Tr(Q)$.

\begin{lemma} \label{thm:bernstein-sum-ops-all}
  Take the setup of \cref{thm:bernstein-sum-ops-gen}
  where each $U_i^a$ is identically $1$:
  we always sample all components of the considered points.

  For any $\rho \in (0, \frac12)$,
  $\lambda \in (0, \rho \norm{Q})$,
  and $\delta \ge 0$,
  it holds with probability at least $1 - \delta$ that
  \[
    \lambda_{\max}\left( Q_\lambda^{-\frac12} \left( Q - \frac1n \sum_{a=1}^n V_a \right) Q_\lambda^{-\frac12} \right)
    \le \frac{2 \beta}{3 n} + \sqrt{\frac{10 d \N'_\infty(\lambda) \beta}{n}}
    ,
    \qquad
    \beta := \log\frac{10 \Tr Q}{\lambda \delta \left(\frac{3}{1 + \rho} - 2 \right)}
  .\]
\end{lemma}
\begin{proof}
Special case of either \cref{thm:bernstein-sum-ops-bernoulli} with $p = 1$
or \cref{thm:bernstein-sum-ops-fixed-num} with $\ell = d$.
\end{proof}

\subsection{Results on Hilbert space operators}
\cref{thm:proj-norm,thm:eigs} were proven and used by \citet{lessismore}.

\begin{lemma}[Proposition 3 of \citet{lessismore}] \label{thm:proj-norm}
  Let $\h_1$, $\h_2$, $\h_3$ be three separable Hilbert spaces,
  with $Z : \h_1 \to \h_2$ a bounded linear operator
  and $P$ a projection operator on $\h_1$ with $\range P = \overline{\range Z^*}$.
  Then for any bounded linear operator $F : \h_3 \to \h_1$
  and any $\lambda > 0$,
  \[
    \norm{(I - P) F} \le \sqrt\lambda \norm{(Z^* Z + \lambda I)^{-\frac12} F}
  .\]
\end{lemma}

\begin{lemma}[Proposition 7 of \citet{lessismore}] \label{thm:eigs}
  Let $\h$ be a separable Hilbert space,
  with $A, B$ bounded self-adjoint positive linear operators on $\h$
  and $A_\lambda = A + \lambda I$,
  $B_\lambda = B + \lambda I$.
  Then for any $\lambda > 0$,
  \[
    \norm{A_\lambda^{-\frac12} B^\frac12}
    \le \norm{A_\lambda^{-\frac12} B_\lambda^\frac12}
    \le (1 - \gamma(\lambda))^{-\frac12}
  \]
  when
  \[
    \gamma(\lambda)
    := \lambda_{\max}\left( B_\lambda^{-\frac12} (B - A) B_\lambda^{-\frac12} \right)
    < 1
  .\]
\end{lemma}

\subsection{Distances between distributions in \texorpdfstring{$\p$}{P}}
\begin{lemma}[Distribution distances from parameter distances] \label{thm:dist-distances}
  Let $f_0, f \in \f$ correspond to distributions $p_0 = p_{f_0}, p = p_f \in \p$.
  Under \cref{assump:bounded-kernel},
  we have that for all $r \in [1, \infty]$:
  \begin{align}
       \norm{p - p_0}_{L^r(\Omega)}
  &\le 2 \kappa e^{2 \kappa \norm{f - f_0}_\h} e^{2 \kappa \min(\norm{f}_\h, \norm{f_0}_\h)}
       \norm{f - f_0}_\h \, \norm{q_0}_{L^r(\Omega)}
\\     \norm{p - p_0}_{L^1(\Omega)}
  &\le 2 \kappa e^{2 \kappa \norm{f - f_0}_\h} \norm{f - f_0}_\h
\\     \mathrm{KL}(f \| f_0)
  &\le c \kappa^2 \norm{f - f_0}_\h^2 e^{\kappa \norm{f - f_0}_\h} (1 + \kappa \norm{f - f_0}_\h)
\\     \mathrm{KL}(f_0 \| f)
  &\le c \kappa^2 \norm{f - f_0}_\h^2 e^{\kappa \norm{f - f_0}_\h} (1 + \kappa \norm{f - f_0}_\h)
\\     h(f, f_0)
  &\le \kappa e^{\frac12 \norm{f - f_0}_\h} \norm{f - f_0}_\h
  \end{align}
  where $c$ is a universal constant
  and $h$ denotes the Hellinger distance $h(p, q) = \norm{\sqrt{p} - \sqrt{q}}_{L^2(\Omega)}$.
\end{lemma}
\begin{proof}
  First note that
  \[
    \norm{f - f_0}_\infty
    = \sup_{x \in \Omega} \abs{f(x) - f_0(x)}
    = \sup_{x \in \Omega} \abs{\langle f - f_0, k(x, \cdot) \rangle_\h}
    \le \kappa \norm{f - f_0}_\h
  .\]
  Then, since each $f \in \h$ is bounded and measurable,
  $\p_\infty$ of Lemma A.1 of \citet{infdim} is simply $\p$,
  and the result applies directly.
\end{proof}

\end{appendices}

\printbibliography
\end{refsection}

\end{document}